\documentclass{article}

\def\shownotes{1}
\usepackage{iclr2026_conference}
\iclrfinalcopy
\usepackage{times}

\usepackage{nicefrac}
\usepackage{xfrac}

\usepackage{macro}
\usepackage[us,12hr]{datetime} %
\usepackage{graphicx, subcaption}
\usepackage{multirow}

\usepackage[utf8]{inputenc}

\usepackage{graphicx} %

\makeatletter
\newcommand{\printfnsymbol}[1]{%
\textsuperscript{\@fnsymbol{#1}}%
}
\makeatother

\title{Efficient Estimation of Kernel Surrogate Models for Task Attribution}
\author{\textbf{Zhenshuo Zhang}\thanks{Email Address: \{zhang.zhens,~duan.mi,~ho.zhang\}@northeastern.edu.}\\
	{Northeastern University}
	\and
\textbf{Minxuan Duan}\printfnsymbol{1}\\
	{Northeastern University}
	\and
\textbf{Hongyang R. Zhang}\printfnsymbol{1}\\
	{Northeastern University}
}

\begin{document}

\maketitle

\begin{abstract}
Modern AI agents such as large language models are trained on diverse tasks---translation, code generation, mathematical reasoning, and text prediction---simultaneously. A key question is how to quantify the influence of each individual training task on performance on a target task, a problem we refer to as \textit{task attribution}. The direct approach, leave-one-out retraining, measures the effect of removing each task, but is computationally infeasible at scale. An alternative approach that builds surrogate models to predict the performance on a target task for any subset of training tasks has emerged in the recent literature. Prior work focuses on linear surrogate models, which capture first-order relationships but miss nonlinear interactions such as synergy, antagonism, or XOR-type effects. In this paper, we first consider a unified task-weighting framework for analyzing task-attribution methods and establish a new connection between linear surrogate models and influence functions via a second-order analysis. Then, we introduce \textit{kernel surrogate models}, which more effectively represent second-order task interactions. To efficiently learn the kernel surrogate, we develop a gradient-based estimation procedure that leverages a first-order approximation of pretrained models; empirically, this yields accurate surrogate estimates with less than $2\%$ relative error without repeated retraining. Experiments across multiple domains---including mathematical reasoning in transformers, in-context learning, and multi-objective reinforcement learning---demonstrate the effectiveness of kernel surrogate models. They achieve a $25\%$ higher correlation with the leave-one-out ground truth than linear surrogates and influence-function baselines, enabling more accurate and scalable task attribution. When used for downstream task selection, kernel surrogate models further yield a $40\%$ improvement in demonstration selection for in-context learning and multi-objective reinforcement learning benchmarks.
\end{abstract}

\section{Introduction}

Modern AI systems are trained to perform well across a diverse set of tasks, ranging from machine translation and code generation to object recognition and mathematical reasoning.
This broad, multi-task capability raises a central question in interpretability: \emph{how do individual training tasks contribute to model performance}?

In this work, we study \emph{task attribution}, the problem of quantifying the influence of each training task on downstream behavior.
Beyond being conceptually fundamental, task attribution has direct implications across several application domains.
In multi-task learning, understanding task relationships can guide the design of neural architectures and loss-reweighting strategies \citep{wuunderstanding,yang2025precise}.
In multi-group learning~\citep{deng2024multi}, it can reveal how training on different demographic groups shapes model behavior. 
In in-context learning \citep{garg2022can,zhang2025linear}, it parallels the question of how adding or removing a single demonstration affects predictions \cite{min2022rethinking}.
Related questions also arise in multi-objective reinforcement learning \citep{yu2020metaworld,zhang2025scalable}, where one wants to understand how competing reward signals influence learned policy.
Accurate task attribution can also help understand when the representations generated by pre-trained large-scale models can be used in downstream tasks \citep{wu2023connecting}.

A natural estimator of task influence is leave-one-out (LOO) retraining, which measures the change in performance when one task is withheld from training.
However, if the training set contains $K$ tasks, computing LOO scores requires $K+1$ full training runs, which is prohibitive.
More efficient proxies for the LOO scores use influence functions~\citep{koh2017understanding,grosse2023studying}, which estimate how model predictions would change if a training sample were added or removed.
These methods require evaluating Hessian-vector products, and are usually applied at the terminal state once the model has memorized the training data \citep{brown2021memorization}.
However, computing influence scores for all $K$ tasks still requires repeatedly Hessian approximations \citep{basu2021influence,kwon2024datainf}.
A complementary line of work---datamodels \citep{ilyas2022datamodels} and data attribution \citep{park2023trak, bae2024training}---builds surrogate functions that approximate black-box model behaviors.
These methods treat the training pipeline as a function that maps data (or task) subsets to predictions, and learn a separate surrogate to emulate this mapping. %
Once trained, the surrogate provides a faster way to evaluate how subset combinations contribute to a given test distribution.
Recent work has primarily focused on \emph{linear surrogate models}, which enjoy appealing theoretical properties, as they can capture first-order effects and admit linear sampling complexities in multi-task settings~\citep{liidentification,li2024scalablemultitask,zhang2025linear}.
However, linear surrogates fail to capture second-order interactions where task effects are non-additive (See illustration in Figure \ref{fig_compare_linear_kernel}).

In this paper, we revisit the surrogate modeling objective through a unified task-weighting framework.
We begin by analyzing linear surrogate models using a second-order Taylor expansion in the weight space.
We find that linear surrogate models estimated via minimizing the surrogate modeling objective are approximately equal to the influence functions, up to third-order expansion errors (Proposition \ref{prop_second_order_datamodels}).
In particular, when the second-term expansion error is small, then linear surrogates and influence functions are approximately the same (Corollary \ref{cor_if_approx}).
A key technical tool is to apply the delta method to certain covariance statistics in the surrogate regression.

Motivated by these observations, we propose \emph{kernel surrogate models} \hl{(\algo)}, which extend surrogate modeling beyond the additive structure inherent in linear methods.
Kernel surrogates--- such as those based on radial basis function (RBF) kernels---naturally capture geometric relationships in the $0$-$1$ subset space. %
A direct kernel estimation, however, would require computing model outcomes on many task subsets.
Instead, we design an efficient estimator that uses \emph{gradients as features}, leveraging a first-order approximation of model outputs with respect to its parameters.
Empirically, we find that the accuracy of this first-order approximation is under $1\%$ relative error across diverse datasets, including CIFAR-10, modular arithmetic, in-context learning, and multi-objective RL benchmarks, on models ranging from small MLP classifiers to language models with up to 34B parameters. Theoretically, we bound the error of this estimate under the assumption that the errors are small (Proposition \ref{prop_jl}).

We evaluate the kernel surrogate modeling framework across a broad range of datasets and architectures. On modular arithmetic reasoning tasks---where the input-output mapping can be XOR, division, quadratic---kernel surrogates improve attribution accuracy by up to $42\%$ relative to existing approaches \citep{liidentification,ilyas2022datamodels,park2023trak}.
We also observe similar gains on in-context learning and sequential decision-making tasks.
Applied to the Qwen3-8B model on sentiment classification and mathematical reasoning, kernel surrogates improve attribution accuracy by $18\%$ over prior methods. %
In reinforcement learning with soft actor-critic on the Meta-World MT10 benchmark \citep{yu2020metaworld}, our method achieves a $5\%$ improvement in environments with dynamically shifting data distributions.
Overall, kernel surrogates consistently outperform prior methods across all settings, increasing correlation with leave-one-out estimates by $25\%$ relative to existing attribution methods.
Finally, for downstream task selection, our approach yields $40\%$ lower loss in both LLM inference (demonstration selection for in-context learning) and multi-objective optimization on Meta-World, while using a runtime comparable to that of fitting linear surrogate models.

Taken together, this paper provides an \emph{efficient estimation algorithm for the task-attribution problem via kernel surrogate models}.
First, we analyze linear surrogate models and show that, under certain assumptions, the influence functions and the estimated regression coefficients are approximately the same.
Second, we introduce \algo{}, along with an efficient gradient-based estimator based on kernel ridge regression.
Thus, our approach is both scalable and flexible for modeling nonlinear task relationships in black-box systems.
Third, we evaluate \algo{} in a variety of empirical settings, including in-context learning, mathematical reasoning, and reinforcement learning.
We provide the code to replicate our empirical results at \href{https://github.com/VirtuosoResearch/Kernel-surrogate-models}{https://github.com/VirtuosoResearch/Kernel-surrogate-models}.

\section{Preliminaries}

We begin by introducing notation and a task-weighting framework. Within this framework, we provide a formal definition of \emph{task attribution}. We then present influence functions and task modeling techniques from the data attribution literature within this unified framework, which also serve as state-of-the-art baselines for evaluating task attribution.

\textbf{Problem setup.} Suppose the training dataset consists of $K$ tasks, denoted by $T_1, T_2, \ldots, T_K$. Each task $k$ contains $n_k$ samples drawn from its task-specific distribution, i.e., $T_k = \{(x_{k,j}, y_{k,j})\}_{j=1}^{n_k}$, for $k = 1, \dots, K$. The total number of training samples is equal to $n = \sum_{k=1}^K n_k$.

We consider a model $f_W$, such as a multitask network, parameterized by $W \in \mathbb{R}^d$ and shared across all tasks. For each task $k$, we denote the task-level loss by $\ell_k(f_W, T_k)$. 
To specify which tasks are included in training, we introduce a $K$-dimensional binary vector $\mathbf{s} = [s_1, \cdots, s_K]^\top$, where $s_k = 1$ indicates that task $k$ is selected in $\mathbf{s}$ and $s_k = 0$ otherwise. The corresponding weighted empirical loss is then defined as 
\begin{align}\label{eq_data_weight}
    \hat{L}(f_{W}, \mathbf{s}) = \left(\sum_{j=1}^K s_j\right)^{-1} \sum_{k=1}^K s_k \ell_k(f_W, T_k),
\end{align}
where $\ell_k(\cdot, T_k)$ denotes the loss function for the $k$-th task, evaluated on a set of samples given by $T_k$.
For example, when $s_k = 1$ for all $k$, the weighted loss $\hat L(f_{W}, \mathbf{s})$ reduces to the average loss over all tasks.
We denote by $\widehat{W}(\mathbf{s}) \leftarrow \argmin_{W} \hat{L}(f_{W}, \mathbf{s})$ the minimizer of the weighted loss, and by $f_{\widehat{W}(\mathbf{s})}$ the corresponding model.

\textbf{Influence estimation.}
We aim to evaluate the influence of training tasks on a target test task $T_{\text{test}}$. To this end, we define a performance metric $F(\mathbf{s}) = \ex{\ell_{\text{test}}(f_{\widehat{W}(\mathbf{s})}, T_{\text{test}})}$, which measures the test loss of the model trained with subset of training tasks indicated by $\mathbf{s}$.

A natural way to quantify the contribution of a task is through leave-one-out (LOO) retraining, which measures the change in performance when task $k$ is excluded from the training set
\[ \mathcal{I}_k^{\text{LOO}} = F\big([\mathbf{1}]_K\big) - F\big([\mathbf{1}]_K - e_k\big), \]
where $[\mathbf{1}]_K \in \real^K$ is the all-ones vector corresponding to training on all tasks, and $e_k \in \{0,1\}^K$ is the basis vector corresponding to the $k$-th coordinate.

\begin{definition}
    The ground-truth task attribution score of task $k$ is given by $\mathcal{I}_k^{\text{LOO}}$, for $k = 1, \dots, K$.
\end{definition}

The other approach to capturing task influence is to compute influence functions on the weighted loss. Given $\mathbf{s}$, the influence of task $k$ on $F(\mathbf{s})$ is formally given by \citet{koh2017understanding}:
\begin{align} 
    \mathcal{I}_k(\mathbf{s})= \left[\nabla_W  F(\mathbf{s})\right]^\top \left[\nabla^2_W \hat L(f_{W}, \mathbf{s})\right]^{-1}  \frac{s_k \nabla_W \ell_k(f_W, T_k)}{\sum_{j=1}^K s_j}. \label{eq_if}
\end{align}
In words, $\cI_k(\mathbf{s})$ quantifies how $\hat L$ changes when task $k$ is infinitesimally up-weighted. The gradient captures task $k$'s direct contribution, and the Hessian inverse accounts for curvature of the loss. The outer product with $\nabla_W  F(\mathbf{s})$ maps the perturbation to its effect on $\hat L$. We provide the derivation of equation \eqref{eq_if} in Appendix \ref{app_proofs}.

\textbf{Linear surrogate models.} An alternative approach to quantify task influence is through task modeling \citep{liidentification}. The key idea is to sample binary vectors $\mathbf{s}^{(1)}, \mathbf{s}^{(2)}, \dots, \mathbf{s}^{(m)}$ from $\{0,1\}^K$, each indicating a random subset of tasks. For each $\mathbf{s}^{(j)}$, we train the model on the corresponding weighted loss to obtain ${\widehat{W}(\mathbf{s}^{(j)})}$ and record its performance $F(\mathbf{s}^{(j)})$. We then fit a linear surrogate model, parameterized by intercept $\alpha \in \mathbb{R}$ and coefficients $\beta = [\beta_1,\ldots,\beta_K]^\top \in \mathbb{R}^K$, by minimizing
\begin{align}\label{eq_datamodels}
    R(\alpha,\beta) = \mathbb{E}_{\mathbf{s}\sim \mathcal{D}}\Big[\big(F(\mathbf{s}) - \alpha - \beta^\top \mathbf{s}\big)^2\Big] \, ,
\end{align}
where $\mathcal{D}$ is a distribution of $\mathbf{s}$ such as the Binomial distribution.
The optimal linear surrogate model is obtained by minimizing the empirical loss $\hat R(\alpha,\beta)$ on $m$ sampled pairs $\{(\mathbf{s}^{(j)}, F(\mathbf{s}^{(j)}))\}_{j=1}^m$.

\section{Methods}

In this section, we present a kernel-based surrogate modeling approach for influence estimation. First, we analyze surrogate modeling alongside influence functions via a second-order Taylor expansion in the task-weighting space. We develop a second-order approximation of the linear surrogate model coefficients and find that it is approximately equal to the influence function.
Specifically, linear surrogate models provide estimates of the influence functions when second-order task interactions are nearly zero in the composition.
Second, we design a kernel-based surrogate modeling procedure to learn the nonlinear interactions between tasks.
Finally, we introduce an efficient gradient estimation algorithm to efficiently learn the kernel surrogate.

\subsection{Analyzing surrogate modeling in the attribution space}

\textbf{Second-order analysis of linear surrogate modeling.}
We center our analysis around $\mathbf{s}^\star = [\mathbf{1}]_K$, which corresponds to the uniform weight applied to all tasks. Then, we characterize task contributions with respect to this global model.  
Expanding around the anchor point $\mathbf{s}^\ast$, we approximate the performance function $F(\mathbf{s})$ using a second-order Taylor expansion as follows:
\begin{align}\label{eqn:second-order}
    F(\mathbf{s}) \approx F(\mathbf{s}^\star) + \underbrace{[\nabla_\mathbf{s} F(\mathbf{s}^\star)]^\top(\mathbf{s}-\mathbf{s}^\star)}_{\text{First-order linear effects}} + \underbrace{\frac{1}{2}(\mathbf{s} - \mathbf{s}^\star)^\top \mathbf{H}_\mathbf{s} (\mathbf{s} - \mathbf{s}^\star)}_{\text{Second-order interaction effects}},
\end{align}
where the Hessian matrix $\mathbf{H}_\mathbf{s} = \nabla_\mathbf{s}^2 F(\mathbf{s}^\star)$ captures the nonlinear interactions between tasks.

We next connect this expansion to linear task modeling, which approximates $F(\mathbf{s})$ with linear surrogates of the form $\alpha + \beta^\top \mathbf{s}$. 
To analyze the behavior of linear surrogates, we apply \emph{a second-order expansion of $F(\mathbf{s})$ in equation \eqref{eqn:second-order} into the surrogate modeling objective, $\hat R(\alpha, \beta)$}.
The following proposition formalizes the resulting characterization of the linear regression coefficients from minimizing $\hat R(\alpha, \beta)$ to get the regression coefficients $\hat \alpha, \hat \beta$.

\begin{proposition}\label{prop_second_order_datamodels}
Let $F:\{0,1\}^K\to\mathbb{R}$ and $\mathbf{s}^\star = [\mathbf{1}]_K$.
Assume the third (partial) derivatives of $F(\cdot)$ are bounded by $c_3$ for some small enough $c_3 > 0$.
Suppose each coordinate in $\mathbf{s}$ is drawn independently from a Bernoulli distribution with probability $p$ for some fixed $p \in (0, 1)$.
Let $\hat \alpha, \hat{\beta}$ be the minimizer of the linear surrogate objective $\hat R(\alpha,\beta)$ on $m$ random samples $\bs^{(1)}, \dots, \bs^{(m)}$.
Then, with probability $1 - \delta$ over the randomness of $m$ sampled subsets, we have
\begin{align}\label{eq_beta-solution-detailed}
    \bignorm{\hat\beta - \nabla_\mathbf{s} F(\mathbf{s}^\star) -
    {(p-1)\mathbf{H}_{\mathbf{s}} [\mathbf{1}]_K} -
    {\tfrac{1 - 2p}{2}{\diag{\mathbf{H}_{\mathbf{s}}}}}} 
    \lesssim c_3 K^{\frac 3 2} p^{-1} + \sqrt{\frac {K + \log(\delta^{-1})} m}.
\end{align}
\end{proposition}
This result shows that the surrogate coefficients $\hat{\beta}$ recover the gradient of $F(\mathbf{s}^{\star})$ with respect to the weights up to two terms induced by the Binomial distribution.
The first corresponds to a sampling shift due to the mean of the Bernoulli draws, while the second reflects the variance in task inclusion across the random subsets.
The key idea is based on the \emph{second-order delta method applied to certain covariance statistics}.
A detailed proof of this proposition is in Appendix~\ref{proof_second_order_datamodels}.

\textbf{Connection between linear surrogate models and influence functions.}
This result also shows the connection between the linear task modeling and influence functions. Specifically, under a first-order approximation ($\bignorms{\mathbf{H}} \le c_2$), both methods estimate the same underlying quantity $\nabla_\mathbf{s} F(\mathbf{s}^\star)$. Under this assumption, the second-order bias terms in equation \eqref{eq_beta-solution-detailed} vanish, showing that the linear task modeling coefficients approximate the gradient of the performance landscape: $\hat{\beta} \approx \nabla_{\mathbf{s}} F(\mathbf{s}^\star).$
Meanwhile, one can show that the influence function computes %
$\vec{\mathcal{I}} = [\cI_1(\mathbf{s}^{\star}), \cI_2(\mathbf{s}^{\star}), \dots, \cI_K(\mathbf{s}^{\star})]^{\top} = \nabla_{\mathbf{s}} F(\mathbf{s}^\star).$
We summarize this discussion precisely as follows.

\begin{corollary}\label{cor_if_approx}
    In the setting of Proposition \ref{prop_second_order_datamodels}, assume further that the second-order task interactions are near zero, i.e., %
    the spectral norm of $\bignorms{\mathbf{H}_{\mathbf{s}}} \le c_2$ for some small $c_2 > 0$, the linear surrogate model coefficients satisfy that with probability at least $1 - \delta$ over the randomness of $\bs^{(1)},\dots,\bs^{(m)}$ for any $\delta > 0$:
    \begin{align}\label{eq_cor_if}
        \bignorm{\hat{\beta} - \vec{\mathcal{I}}} \le c_2 (\sqrt K + 1/2) + O\left(c_3 K^{3/2} p^{-1}\right) + O\left(\sqrt{\frac{K + \log(\delta^{-1})}{m}}\right).
    \end{align}
\end{corollary}

Thus, provided that $c_2, c_3$ are small enough and $m \ge O(K + \log(\delta^{-1}))$, then influence functions can approximate linear surrogate models, while linear surrogate models capture additive relations. %
In particular, more expressive attribution methods are needed to express second-order task interactions.
We provide the proof of this result in Appendix \ref{proof_cor_if_approx}.

Next, we experiment with a binary classification to empirically validate the connection, where the first-order approximation holds. 
We compare the estimates from both the linear task model and influence functions against the ground-truth leave-one-out (LOO) retraining scores. 
In Figure~\ref{fig_correlation_datamodels_if}, we find that both methods produced estimates that closely aligned with the ground truth: the Linear Task Model and Influence Functions achieved Pearson correlations of $0.98$ and $0.97$, respectively, with the LOO scores.
The estimates from the two methods are also highly correlated, with a Pearson correlation of $0.96$.
This demonstrates that when the performance landscape is nearly linear, both the empirical and analytical approaches indeed converge to the same underlying attribution scores. A detailed description of the experimental setup is available in Appendix~\ref{app_correlation_datamodels_if}.

\begin{figure}[t!]
    \centering
    \begin{subfigure}[b]{0.33\textwidth}
        \centering
        \includegraphics[width=0.80\textwidth]{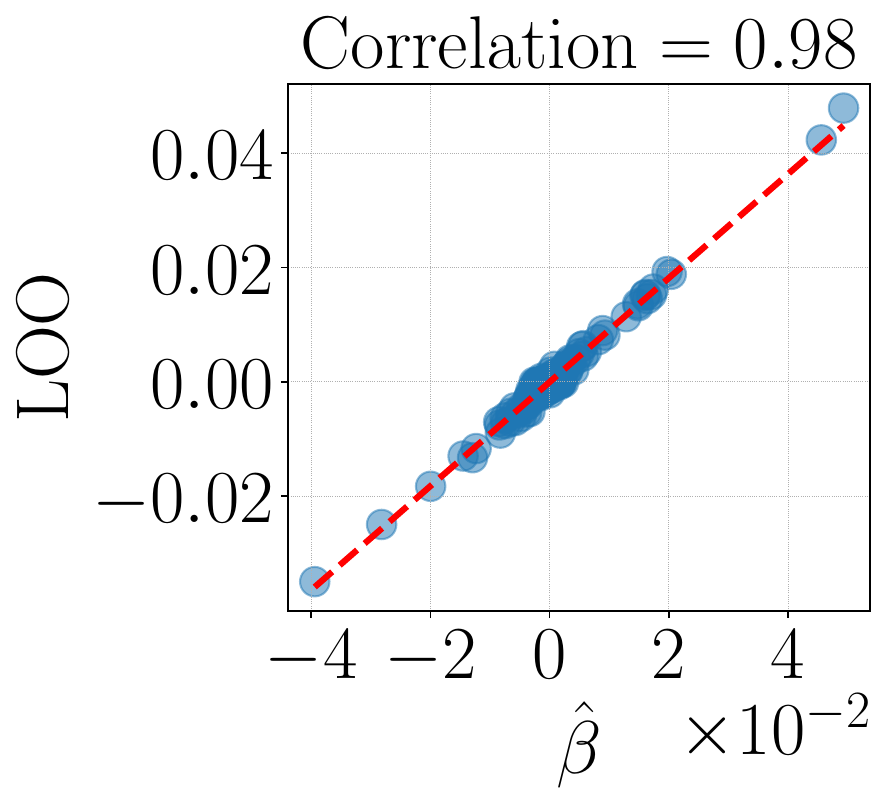}
        \caption{LOO / Linear Surrogate}\label{fig_datamodels_loo}
    \end{subfigure}\hfill
    \begin{subfigure}[b]{0.33\textwidth}
        \centering
        \includegraphics[width=0.80\textwidth]{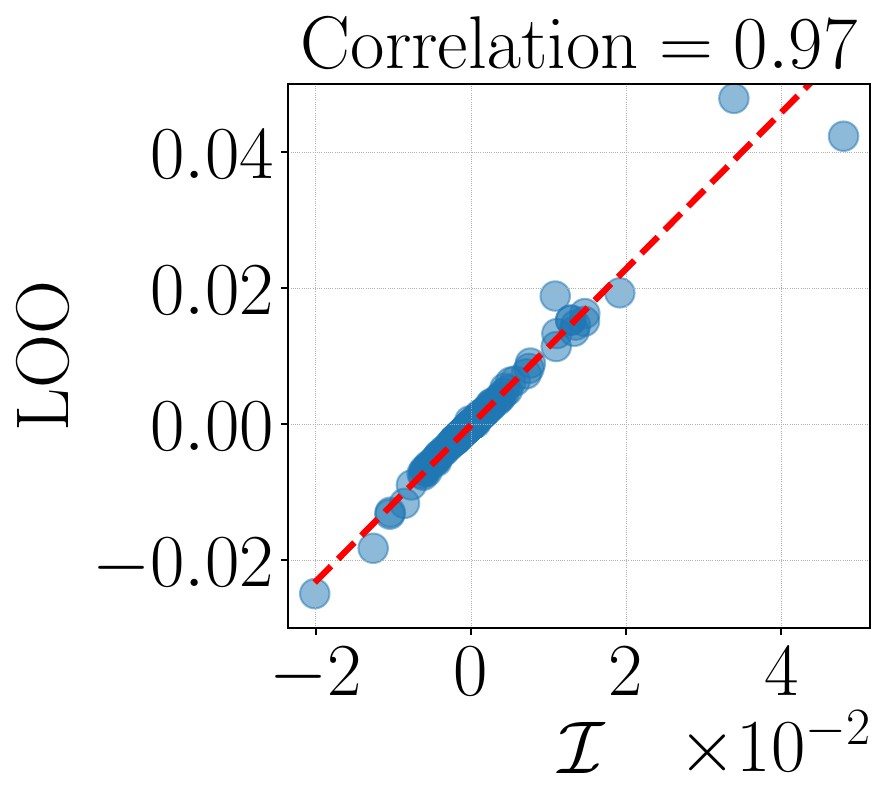}
        \caption{LOO / IF}\label{fig_IF_loo}
    \end{subfigure}\hfill
    \begin{subfigure}[b]{0.33\textwidth}
        \centering
        \includegraphics[width=0.80\textwidth]{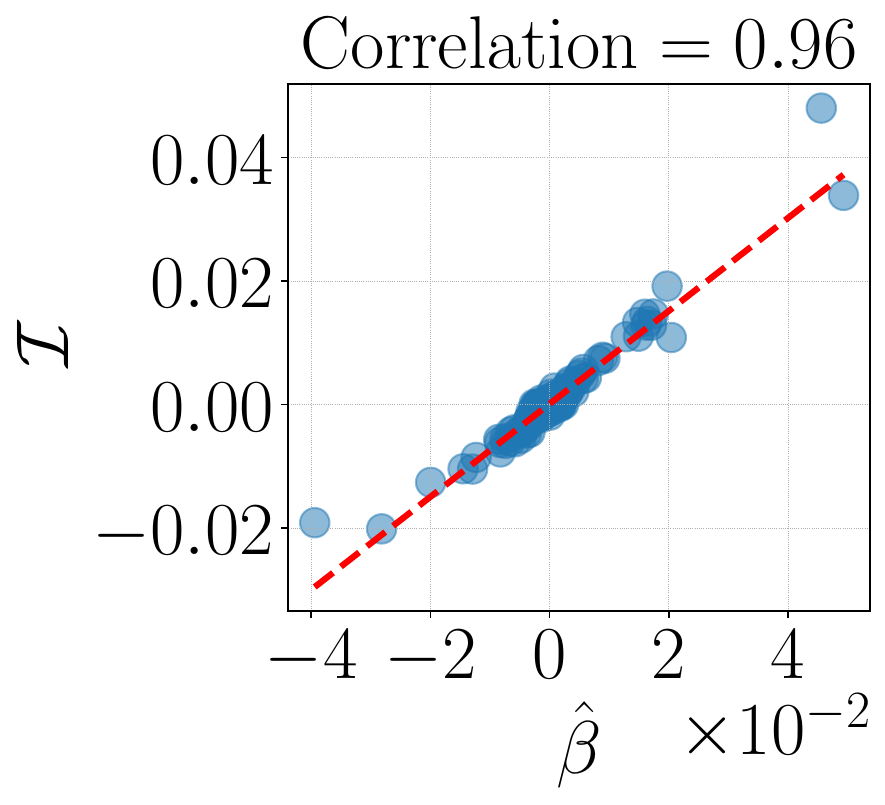}
        \caption{IF / Linear Surrogate}\label{fig_datamodels_IF}
    \end{subfigure}
    \caption{We compare influence functions (IF), leave-one-out (LOO) retraining, and linear surrogate models. Each point corresponds to the effect of removing a single training example.}\label{fig_correlation_datamodels_if}
\end{figure}

\subsection{Learning kernel surrogate models}

To build intuition for why a nonlinear data model is necessary, we present a numerical example in Figure~\ref{fig_compare_linear_kernel}.
We use a binary classification task with a two-layer MLP as the base classifier. 
The final goal of the surrogate model is to predict the MLP's output for a fixed test sample, given the subset of the training data it was trained on. We specifically analyze the effect of different subsets of training data sampled from near the MLP decision boundary. The detailed setting is in Appendix~\ref{app_correlation_datamodels_if}.

\begin{wrapfigure}[9]{r}{0.385\textwidth}
  \vspace{-0.200in}
  \begin{center}
    \includegraphics[width=0.380\textwidth]{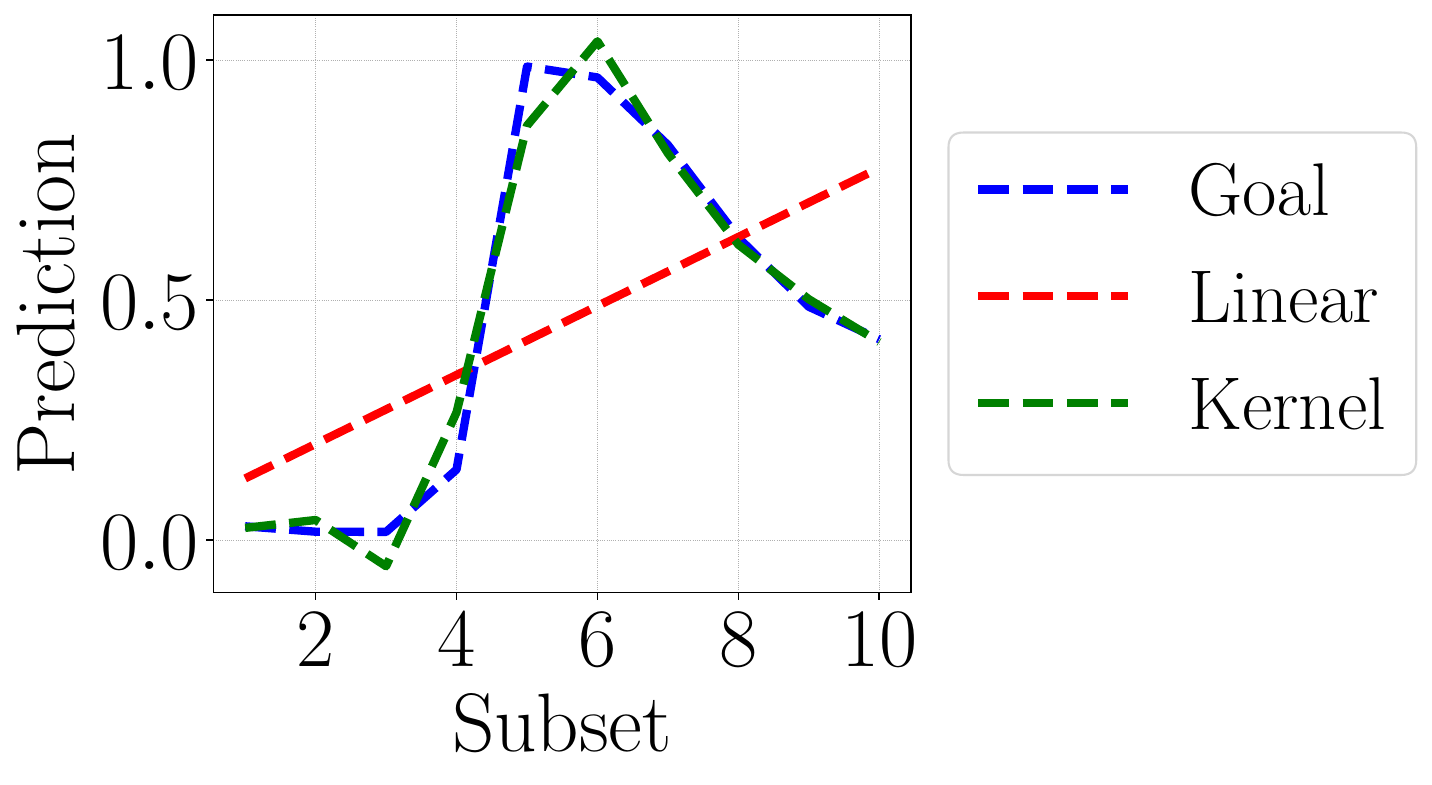}
  \end{center}
  \vspace{-0.1850in}
  \caption{Illustrate linear vs. kernel models on a decision boundary.}\label{fig_compare_linear_kernel}
\end{wrapfigure}

The influence of a training subset exhibits strong nonlinearity that cannot be decomposed into individual-sample contributions. When specific samples are combined in a training subset, they cause changes in the MLP's learned function, leading to transitions in test-point predictions. The linear surrogate model does not capture these interactions. By contrast, the RBF kernel surrogate model captures these dependencies by modeling the joint influence of sample combinations. This demonstrates that non-linear surrogate models are necessary for capturing nonlinear interactions.

\textbf{Estimating surrogate models using kernel ridge regression.}
To address the limitations of linear surrogates, we propose kernel surrogate models that replace the linear mapping with a nonlinear function learned via kernel ridge regression.
This enables modeling complex, non-additive task interactions.

We learn a kernel surrogate function $g_\theta: \{0,1\}^{K} \to \mathbb{R} $ within a Reproducing Kernel Hilbert Space (RKHS) $\mathcal{H}$ by minimizing the following objective:
\begin{equation}
  \min_{g_\theta \in \mathcal{H}} \sum_{i=1}^{m} \big( F(\mathbf{\mathbf{s}}^{(i)}) - g_\theta(\mathbf{\mathbf{s}}^{(i)}) \big)^{2} + \lambda \lVert g_\theta \rVert^{2}_\cK,
\end{equation}
where $\lambda > 0$ is a regularization parameter and $\cK$ is an $m \times m$ kernel matrix with entries $\cK_{i,j} = k(\mathbf{\mathbf{s}}^{(i)}, \mathbf{\mathbf{s}}^{(j)})$.
Since input vectors $\mathbf{\mathbf{s}}$ are binary, we use the Radial Basis Function (RBF) kernel:
\begin{equation}
  k(\mathbf{\mathbf{s}}^{(a)},\mathbf{\mathbf{s}}^{(b)})= \exp\big(-\gamma \| \mathbf{\mathbf{s}}^{(a)}-\mathbf{\mathbf{s}}^{(b)} \|^{2}\big).
\end{equation}
By the representer theorem~\citep{scholkopf2001generalized}, the minimizer takes the form: \[ g_\theta(\mathbf{\mathbf{s}})= \sum_{i=1}^{m} \theta_i k\big(\mathbf{\mathbf{s}}^{(i)}, \mathbf{\mathbf{s}}\big). \]
The coefficient vector $\theta = [\theta_{1}, \ldots, \theta_{m}]^{\top}$ is computed as: $\theta = (\cK + \lambda \id)^{-1}\mathbf{F},$
where $\mathbf{F} = [F(\mathbf{\mathbf{s}}^{(1)}), \ldots, F(\mathbf{\mathbf{s}}^{(m)})]^{\top}$ is the vector of observed outcomes and $\id$ denotes the identity matrix.

This kernel provides flexibility for capturing complex relationships while maintaining computational tractability.
Hyperparameters $\lambda$ and $\gamma$ are selected via cross-validation.
Unlike linear models that assign fixed coefficients to individual tasks, the kernel task model learns a global non-linear function that predicts performance for any task combination, thereby capturing intricate inter-dependencies.

We empirically validated this choice by comparing the RBF kernel with linear and polynomial kernels with degrees $\{1,2,3\}$ in Table~\ref{tab_kernels}. We use ResNet-9 network classifiers trained on the CIFAR-10 dataset~\citep{krizhevsky2009learning} as the base model. The detailed description is in Appendix~\ref{app_kernel_comparison}.
We find that the RBF kernel achieves a much lower residual error compared to the linear surrogate model. We discuss more kernels in Appendix~\ref{app_kernel_comparison}.

\begin{figure}[t]
  \centering
  \begin{minipage}[t]{0.485\textwidth}
    \centering
    \captionsetup{type=table}
    \caption{We investigate the surrogate model performance with different kernels.}\label{tab_kernels}
    \begin{tabular}{l|cc}
        \toprule
        Residual error               & Linear   & RBF \\ \hline
        CIFAR-10 & $4.4_{\pm0.9}$  & $1.0_{\pm0.0}$    \\
        Modular arithmetic & $4.6_{\pm1.3}$  & $1.5_{\pm0.4}$   \\
        In-context learning & $0.8_{\pm0.2}$  & $0.4_{\pm0.1}$\\
        Multi-objective RL & $0.2_{\pm0.1}$  & $0.1_{\pm0.1}$\\
        \bottomrule
    \end{tabular}
  \end{minipage}\hfill
  \begin{minipage}[t]{0.485\textwidth}
    \centering
    \captionsetup{type=table}
    \caption{Relative approximation error of $\epsilon_{W}(x)$, tested on four different tasks.}\label{tab_first_order_approx}
    \begin{tabular}{lc}
        \toprule
         & Relative error \\ \hline
         CIFAR-10 & $1.02_{\pm0.69}\%$\\
         Modular arithmetic & $2.40_{\pm2.17}\%$ \\
         In-context learning & $0.51_{\pm0.04} \%$ \\
         Multi-objective RL & $0.43_{\pm0.73}\%$\\
        \bottomrule
    \end{tabular}
  \end{minipage}
  \vspace{-0.1in}
\end{figure}

\begin{remark}
    \hl{A major property of RBF kernels~\citep{poggio2002networks} is its universal expressivity over any distance functions on the space $\set{0, 1}^K$.
    Moreover, the RBF kernel encodes an inductive bias that matches the geometry of the subset space $\set{0, 1}^K$ (See also, the use of heat kernels for community detection \citep{kloster2014heat}).
    We formalize this in Proposition \ref{prop_rbf}, Appendix \ref{app_expressivity}.}
\end{remark}

\subsection{Efficient estimation via projected gradients}

One downside of kernel surrogate models is that they involve training $m$ separate models to obtain $F( \bs^{(i)} )$, one on each $\bs^{(i)}$.
Rather than pursuing repeated training, we introduce an efficient algorithm to approximate the model's outcomes.
Consider a first-order Taylor expansion of model outputs around the initialization $W_0$ as follows:
\begin{align}\label{eq_foapp}
    f_{W}(x) = f_{{W}_0}(x) + \inner{\nabla f_{W_0}(x)}{W - W_0} + \epsilon_{W}(x),
\end{align}
where %
$\epsilon_{W}(x)$ is the first-order Taylor expansion error of $f_{W_0}$ on input $x$.
\hl{The intuition around this estimation algorithm can be related to how neural networks learn features through the gradients \citep{malladi2023kernel,radhakrishnan2024mechanism}.}
Empirically, we also evaluate ${\epsilon_{W}(x)}$ relative to ${f_{{W}}(x)}$ in Table~\ref{tab_first_order_approx}, which is with $2\%$ across various test environments, \hl{and report further evaluation of this approximation on LLMs with up to 34B parameters in Table \ref{tab_llm_approx_error}, Appendix \ref{app_foa}.}.

Next, we substitute approximation \eqref{eq_foapp} of $f_W(x)$ into the empirical loss to estimate the optimal fine-tuning parameter perturbation on top of the initialization ${W}_0$ for every $\bs^{(i)}$. %
Specifically, we formulate an approximate objective on each subset $\mathbf{s}^{(i)}$ as:
\[ Z_{\mathbf{s}^{(i)}}^\star = \argmin_{Z \in \mathbb{R}^d} \sum_{k=1}^K \frac{s_k^{(i)} }{\sum_{i=1}^K s_k^{(i)}} \sum_{(x, y) \in T_k} \ell_k \left( f_{{W}_0}(x) + \inner{{\nabla_W f_{W_0}}(x)} {Z}, y \right), \]
assuming the expansion error $\epsilon_{W}(x)$ is negligibly small.
This corresponds to a (regularized) multinomial logistic regression problem, with (projected) gradients serving as features.

Based on this estimation algorithm, if we want to minimize the approximate loss on all the subsets $\mathbf{s}^{(1)}, \mathbf{s}^{(2)}, \dots, \mathbf{s}^{(m)}$, we only need to compute the functional values and the gradients at the initialization ${W}_0$ (such as a pretrained LLM) once, including $f_{{W}_0}(x)$ and $\nabla f_{{W}_0}(x)$ for all $x$ in the training dataset.
\emph{This eliminates the need for repeated training and dramatically speeds up the estimation procedure, and is practical even for very large $m$ since the estimation procedure can be executed entirely on CPUs.}

Finally, we estimate the model output on a test sample $x, y$ using the same linear approximation:
\[ \hat{f}(x) = f_{{W}_0}(x) + \inner{\nabla f_{{W}_0}(x)} {Z_{\mathbf{s}^{(i)}}^\star}. \]
The test performance of each $\bs^{(i)}$ is then evaluated as the loss $\ell(\hat{f}(x), y)$ on the test dataset.
Another challenge in this procedure is that the gradient vectors can have many features.
We address this problem by applying Gaussian random convolutions to project the gradient vectors into much lower dimensions.
Then, solving the regression problem in the reduced space takes only several seconds and can be run quickly on many subsets $m$.
Due to space limit, we summarize the detailed procedure in Algorithm \ref{alg:compute_rp}, Appendix \ref{app_alg}.
Taken together, we summarize the overall procedure in Algorithm \ref{alg_krr}.

\hl{\textbf{Performance guarantees.} Assuming that the first-order approximation is small, we can bound the accuracy of the above estimation procedure relative to the true test performance using the Johnson-Lindenstrauss lemma. For details, see the result statement and its proof following Proposition \ref{app_foa}.}

\begin{algorithm}[t!]
\caption{Estimating Kernel Surrogate Models (\algo)}\label{alg_krr}
\textbf{Input:} $K$ training tasks $T_1, \ldots, T_K$, pretrained model $f_{W_0}$, test dataset and metric $F(\cdot)$\\
\textbf{Requires:} Subset sampling distribution $\cD$, number of subsets $m$, $\ell_2$-regularization parameter $\lambda$, kernel function $k(\cdot, \cdot)$\\
\textbf{Output:} An estimated kernel surrogate model (KSM)
\begin{algorithmic}[1]
\Statex \textit{/*\qquad\qquad\qquad\qquad\quad Generate random subsets and estimate test performance\hfill */}
\State ${D}_{{s}} \leftarrow \emptyset$: Initialize surrogate dataset
\State  $\{f_{{W}_0}(z), \nabla f_{W_0}(z) \} \leftarrow$ Compute logits and gradients for all relevant samples $z$
\For{$i = 1, \ldots, m$}
    \State $\mathbf{s}^{(i)} \sim \cD$ %
    \State $\widehat F(\bs^{(i)}) \leftarrow$ \texttt{GradEx}$(\mathbf{s}^{(i)}, \{f_{{W}_0}(z), \nabla f_{W_0}(z)\}, F)$: Apply Algorithm \ref{alg:compute_rp}
    \State ${D}_{{s}} \leftarrow {D}_{{s}} \cup \{(\mathbf{s}^{(i)}, \widehat F(\bs^{(i)}))\}$
\EndFor
\Statex \textit{/*\qquad\qquad\qquad\qquad\quad Learn a kernel surrogate model using kernel ridge regression\hfill */}
\State $\cK_{i, j} \leftarrow k(\mathbf{s}^{(i)}, \mathbf{s}^{(j)})$: Construct the $m \times m$ kernel matrix $\cK$ 
\State $\mathbf{y} \leftarrow [\widehat F(\bs^{(1)}), \ldots, \widehat F(\bs^{(m)})]^{\top}$: Compute the outcome vector
\State $\theta \leftarrow (\cK + \lambda \cdot \id_m)^{-1} \mathbf{y}$: KRR coefficients
\State \textbf{return} $g_\theta(\mathbf{s}) = \sum_{i=1}^m \theta_i k(\mathbf{s}^{(i)}, \mathbf{s})$
\end{algorithmic}
\end{algorithm}

\section{Experiments}\label{sec_exp}

Our experiments investigate three key questions:
(1) How does \algo{} compare to existing attribution techniques on complex tasks such as modular arithmetic reasoning, in-context learning, and multi-objective reinforcement learning? 
(2) Can the attributions generated by \algo{} be effectively leveraged for downstream task selection? 
(3) How does a complex loss landscape impact the performance of \algo{} compared to methods relying on linear surrogate models?
We evaluate our approach across three settings, including arithmetic reasoning, in-context learning, and multi-objective reinforcement learning.
We find that \algo{} can improve over existing influence estimation methods by $25\%$.
When applied to downstream task selection, \algo{} can improve the performance of baselines by $41\%$.
Various robustness checks and ablation studies validate the consistency of \algo{} under different optimizers and sample sizes.

\subsection{Experiment setup}\label{sec_exp_setup}

\textbf{Datasets and models.}
Our primary controlled environment is an arithmetic reasoning task, where we train a small decoder-only transformer.
Consider learning a transformer model to perform modular arithmetic operations over two numbers, in the form of $a \circ b \pmod{p} =c$, where $\circ$ is an arithmetic operation and $p$ is a prime. Inputs are composed of four tokens: $a$, $\circ$, $b$, and $=$, with $a$ and $b$ generated between $0$ and $p-1$. The label is the result of the arithmetic operation $c$. Specifically, we use the addition task $a + b \pmod{p} =c$ and the quadratic task $a^2 +b^2+ab \pmod{p} =c$, where we set $p=97$.

For more complex scenarios, we evaluate a Qwen3-8B model on in-context classification tasks. We construct prompts with $k=4$ in-context examples followed by a query. For attribution computation, we treat each example's input embedding as the sample input, and the query's cross-entropy loss as the model score $F(\mathbf{s})$. We evaluate on two datasets, including SST-2 and coin flip.

We also evaluate attribution in multi-objective reinforcement learning using the Meta-World MT10 benchmark. It contains ten different manipulation tasks. We adopt the Soft Actor-Critic (SAC) algorithm as our training protocol. Each task is treated as a sample for attribution analysis, and we evaluate the model score by the average rewards of all tasks.

\textbf{Baselines.}
We compare \algo{} with existing data attribution methods. 
Influence functions~\citep{koh2017understanding} employ implicit differentiation, computing the inverse Hessian via LISSA. 
TracIn~\citep{pruthi2020estimating} traces prediction changes throughout training by computing gradient products between training and test examples. 
Linear surrogate models~\citep{ilyas2022datamodels} train linear models that predict test behavior from binary indicators of training sample inclusion across multiple model retrainings. 
Trak~\citep{park2023trak} applies random projection and the one-step Newton method to approximate the datamodel results without retraining.
\re{SOURCE~\citep{bae2024training} approximates unrolled differentiation by segmenting training into a few stationary phases and using gradient/Hessian statistics from a handful of checkpoints to estimate counterfactual parameter changes.
Bayesian influence functions (BIF)~\citep{kreer2025bayesian} extend classical influence functions by expressing influence as a covariance under a localized Bayesian posterior and estimating it via batched SGLD sampling.}

\textbf{Evaluation.}
We evaluate the data attribution methods by comparing them to the linear datamodeling score (LDS)~\citep{ilyas2022datamodels}. 
We compute LDS in the following steps:
(1) Sample $m$ fixed subsets of the training set, which are represented by $\mathbf{s}^{(i)},\dots,\mathbf{s}^{(m)}\in\{0,1\}^K$. We sample $\mathbf{s}$ from a Bernoulli distribution with probability $p$.
(2) For each $\mathbf{s}^{(i)}$, we train a model $F(\mathbf{s}^{(i)})$, and use a task attribution method to obtain an estimation $\hat{F}(\mathbf{s}^{(i)})$.
(3) We compute the LDS from the Spearman correlation $\rho$ between the real model output and the task attribution estimation. %

\subsection{Experimental results}

\textbf{\re{Attribution} results.} Our results \re{in Table~\ref{tab_attribution}} demonstrate that by moving beyond the linear assumptions of prior work, \algo{} achieves a more accurate estimation of task attribution. The advantage is most pronounced in tasks with complex, non-linear structures. For example, in modular arithmetic reasoning, \algo{} improves the LDS by over $42\%$ on average compared to linear surrogate models. 
Similarly, on the reasoning Coin Flip ICL task, where the language model exhibits intricate prompt-following behavior that falters linear approximations, our method's ability to capture higher-order interactions yields a $25\%$ relative improvement in LDS.
In multi-objective reinforcement learning, where the task's more transparent structure allows the linear baseline to perform well already (LDS of $0.76$), \algo{} still shows an improvement to $0.80$.

\re{We further evaluate whether \algo{} scales to scenarios involving a large number of tasks. In our in-context learning experiments, we use $500$ prompt samples to simulate a task-scalable setting. The resulting LDS values remain comparable to the initial results tested on $50$ tasks: $0.31$ on SST-2 and $0.53$ on Coin-Flip.
These results indicate that the performance of \algo{} does not degrade as $K$ increases, demonstrating its robustness in scaling to settings with large $K$.}

\textbf{Task selection results.}
Next, we show that the accurate attributions from \algo{} are applicable to downstream optimization tasks.
We apply task attribution methods to find optimal groupings for downstream performance. For prompt selection in ICL, we select the top-4 examples with the highest attribution score. For synergistic task selection in RL, we co-train each target task with the top-3 tasks identified by our method to encourage positive transfer.

To derive an individual attribution for each item from the kernel surrogate model, we adopt a random ensemble method. We sample multiple subsets and use our kernel surrogate to obtain a performance prediction for each. The final attribution for any given item is then calculated as the average prediction score of all sampled subsets in which that item was included.

As shown in Table~\ref{tab_task_selection}, \algo{} achieves a $40\%$ lower loss in the prompt selection task in ICL, and improves the target rewards by $15\%$. These results demonstrate the benefit of using \algo{} over linear surrogate models.

\begin{table}[t]
\centering
\caption{Summary of comparison results on task attribution. We report the mean and standard deviation from five independent runs.}\label{tab_attribution}
\resizebox{\textwidth}{!}{
\begin{tabular}{l|cc|cc|c}
\toprule
\multirow{2}{*}{Methods} & \multicolumn{2}{c|}{Modular arithmetic reasoning} & \multicolumn{2}{c|}{In-context learning} & \multicolumn{1}{c}{Multitask RL} \\
              & Addition    & Quadratic & SST-2 & Coin flip & Metaworld \\ \hline
Influence functions            & $0.03_{\pm0.01}$ & $0.01_{\pm0.01}$ & $0.16_{\pm0.05}$ & $0.05_{\pm0.10}$ & $0.71_{\pm0.11}$ \\
TracIn        & $0.17_{\pm0.03}$ & $0.32_{\pm0.09}$ &$0.21_{\pm0.05}$& $0.32_{\pm0.09}$ & $0.45_{\pm0.15}$ \\
TRAK          & $0.14_{\pm0.01}$  & $0.28_{\pm0.06}$ & $0.11_{\pm0.08}$ & $0.23_{\pm0.12}$ & $0.42_{\pm0.19}$ \\
\re{SOURCE} & \re{$0.22_{\pm0.05}$} & \re{$0.28_{\pm0.06}$} & \re{/} & \re{/} & \re{$0.27_{\pm0.19}$ } \\
\re{BIF} & \re{$0.18_{\pm0.16}$} & \re{$0.42_{\pm0.07}$} & \re{$0.14_{\pm0.02}$} & \re{$0.28_{\pm0.06}$} & \re{$0.22_{\pm0.24}$} \\
Linear surrogate models & $0.18_{\pm0.02}$ & $0.44_{\pm0.09}$ & $0.33_{\pm0.05}$ &$0.43_{\pm0.05}$ & $0.76_{\pm0.04}$ \\
\algo{}          & $\mathbf{0.30}_{\pm0.12}$ & $\mathbf{0.52}_{\pm0.08}$ & $\mathbf{0.37}_{\pm0.02}$ & $\mathbf{0.54}_{\pm0.01}$ & $\mathbf{0.80}_{\pm0.04}$          \\
\bottomrule
\end{tabular}}
\end{table}

\begin{table}[t!]
\centering
\caption{We evaluate \algo{} on the downstream task selection. We measure performance by the loss in ICL tasks and the target rewards in multi-objective RL tasks. We measure the running time in minutes.}\label{tab_task_selection}
\begin{tabular}{lccc|lcc}
    \toprule
    Loss $(\downarrow)$ & SST-2 & Coin flip & Time $(\downarrow)$ & Rewards $(\uparrow)$ & MT 10 & Time $(\downarrow)$ \\ \hline
    Influence func.  & $0.23_{\pm0.21}$ & $1.62_{\pm1.33}$ & $17_{\pm1}$ & Influence func.      & $16.3_{\pm1.4}$ & $2_{\pm1}$\\
    \re{TracIn} & \re{$0.25_{\pm0.40}$} & \re{$0.54_{\pm0.68}$} & \re{$3_{\pm1}$} & \re{TracIn} & \re{$16.0_{\pm2.3}$} & \re{$3_{\pm1}$} \\
    \re{Trak} &  \re{$0.32_{\pm0.48}$}& \re{$0.72_{\pm0.87}$} & \re{$2_{\pm1}$} & \re{Trak} & \re{$15.8_{\pm4.5}$} & \re{$1 _{\pm1}$} \\
    Lin. surrogate  & $0.20_{\pm0.11}$ & $0.05_{\pm0.03}$ & $2_{\pm1}$  & Lin. surrogate     & $13.1_{\pm4.3}$ & $1_{\pm1}$\\
    \algo{} & $0.16_{\pm0.08}$ & $0.02_{\pm0.03}$  & $2_{\pm1}$ & \algo{} & $18.8_{\pm5.6}$ & $1_{\pm1}$\\
    \bottomrule
\end{tabular}
\end{table}

\subsection{Regularization of Hessian}

Recall from equation \eqref{eq_beta-solution-detailed} that the error of surrogate models is influenced by the second-order terms. %
Next, we compare the accuracy of linear task modeling by measuring the curvature of the loss Hessian $\nabla_W^2\hat{L}(\hat{W})$, by controlling it with different Hessian regularization.
We report results on a modular quadratic task trained using a two-layer transformer.
We compared three methods, including the standard SGD without explicit Hessian regularization, and two Hessian regularization methods: sharpness-aware minimization (SAM)~\citep{foret2021sharpnessaware}, and noise stability optimization (NSO)~\citep{zhang2024noise}.
We measure the Hessian trace on the global model, and the LDS and residual error of the linear surrogate model are measured on subsets with $\alpha=0.9$. %

In SGD, the Hessian trace increases continuously throughout the training process. This worsens the linear surrogate model's performance, characterized by a decrease in LDS. Eventually, the LDS dropped to $0.01$. %
By contrast, both SAM and NSO constrain the Hessian trace. This regularization results in improved model fit, whose LDS remained relatively consistent throughout training.

\section{Related Work}

Our work is related to the literature on training data attribution methods. Data attribution methods have been extensively studied to quantify the influence of individual training samples on model behavior~\citep{koh2017understanding,brown2021memorization,ilyas2022datamodels}. In this work, we extend these concepts by measuring the contribution of groups of samples that form different training tasks.

The first line of approach is based on the influence function~\citep{koh2017understanding}. These methods compute the influence function. The main challenge is that the influence function involves Hessian computation.
Instead of computing the exact Hessian inverse, recent approaches use the Gauss-Newton Hessian approximation~\citep{choe2024your} or sparse gradient compression~\citep{hu2025grass}.
Influence functions can also be applied to non-decomposable loss functions, such as contrastive loss and preference loss, allowing for data attribution across a broader class of models~\citep{deng2025a}.
The second line of approach studies the trajectory of model training, rather than just examining the final state of the trained model.
These methods aim to trace the training dynamics and predict the counterfactual trajectory that would result from a different training set~\citep{bae2024training}.
Recent work calculates the exact influence of training data on a single, deterministic model instance by leveraging large-scale meta-gradient computation~\citep{engstrom2025optimizing, ilyas2025magic}.
The third line of research involves using a surrogate model to approximate the behavior of the original complex model and then computing the surrogate model.
Recent work shows that the model output can be linearized in a certain local area~\citep{li2024scalable}. Under this assumption, one can rephrase the model output as a first-order Taylor expansion on the parameters.
This approach provides an efficient estimation for the leave-one-out score~\citep{liidentification} and for datamodels~\citep{park2023trak}.
Building on surrogate models, our work uses a kernel-based model to better capture the non-linear interactions between tasks.
The use of kernel ridge regression, also known as Tikhonov regularization, can be related to earlier literature on kernel methods \citep{dicker2017kernel}, and it may be worth further examine the role of regularization in task attribution. The RBF kernel that we have used is also related to the heat kernel, which is highly effective for seed set expansion in community detection \citep{kloster2014heat}.

A limitation of this work is that our evaluation primarily focuses on prediction tasks. Extending our attribution framework to examine other aspects of language models such as planning abilities remains an open question, which is a fundamental aspect of human intelligence \citep{wang2024alpine}.
In particular, this approach may explain the generalization gaps between SFT and RL \citep{wang2025benefits}; we leave a detailed exploration of this connection to future work.

\begin{figure}[t!]
\centering
    \begin{subfigure}[b]{0.32\textwidth}
        \centering
        \includegraphics[width=0.80\textwidth]{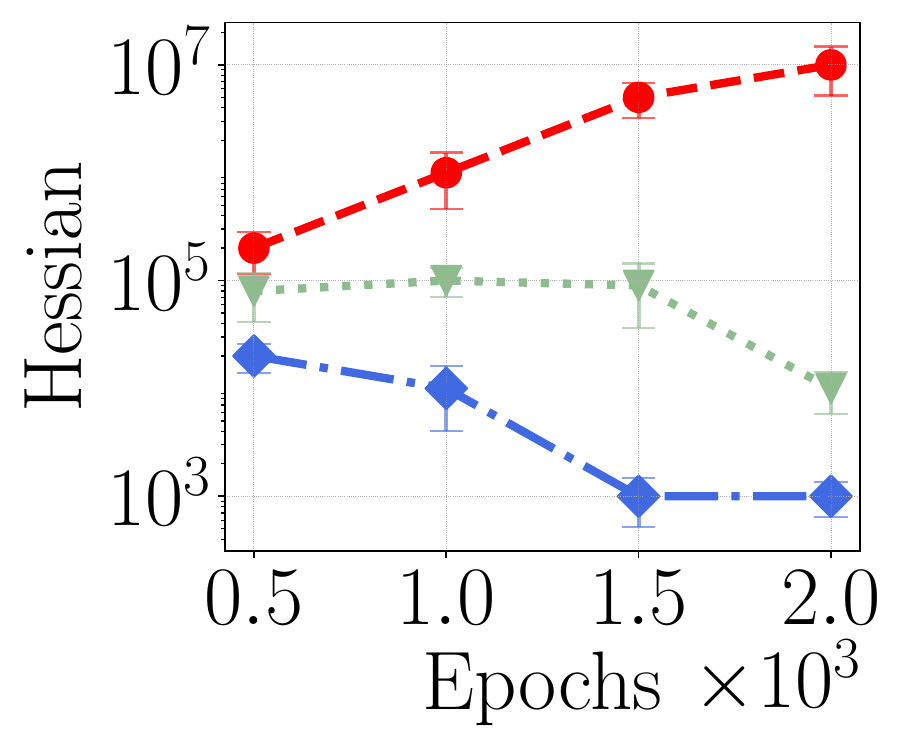}
        \caption{Hessian trace}
    \end{subfigure}%
    \begin{subfigure}[b]{0.32\textwidth}
        \centering
        \includegraphics[width=0.80\textwidth]{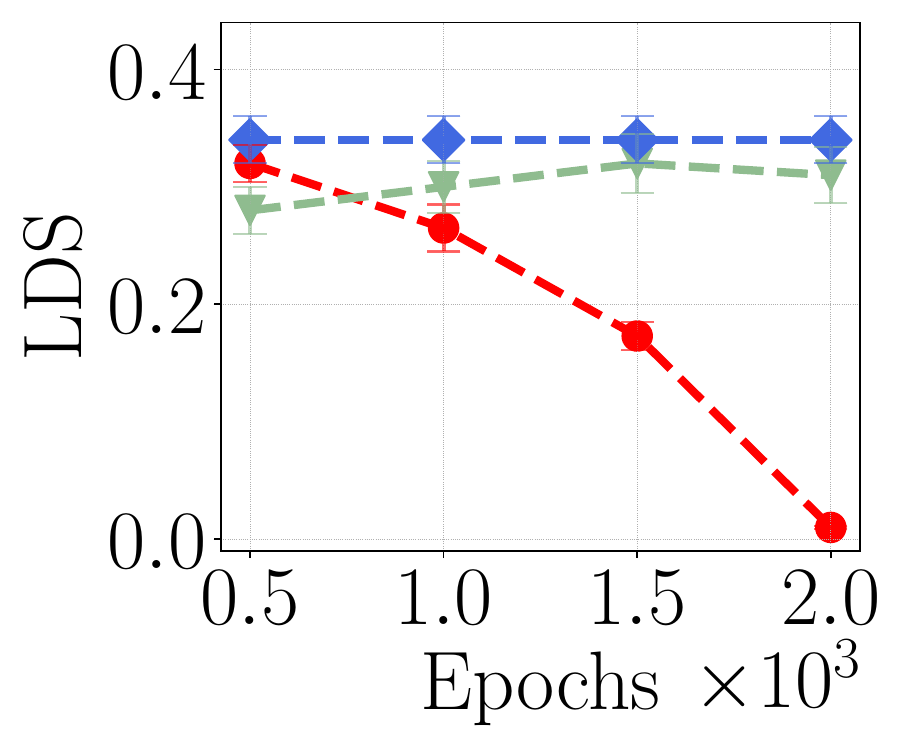}
        \caption{Linear model}
    \end{subfigure}%
    \begin{subfigure}[b]{0.32\textwidth}
        \centering
        \includegraphics[width=0.80\textwidth]{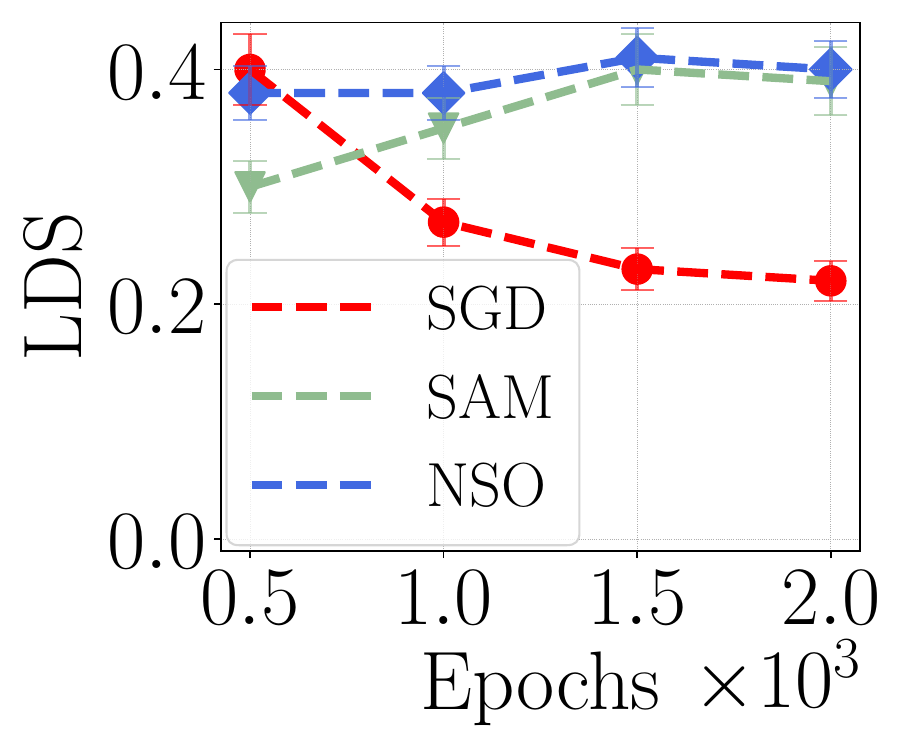}
        \caption{RBF Kernel}
    \end{subfigure}
    \caption{We investigate both linear and kernel surrogate models' fit under different Hessian regularization during model training. The linear surrogate model does not fit model outcomes when using SGD, and only works when its Hessian is regularized. By contrast, kernel surrogate models remain robust when trained with various optimizers and regularization.%
    }\label{fig_ablation_hessianreg}
\end{figure}

\section{Conclusion}

In this paper, we study the problem of task attribution, which aims to quantify the influence of different training tasks on a model's final performance.
We first unify the theories of existing influence functions and linear surrogate models, highlighting their limitations in handling non-linear interactions between tasks.
To overcome this challenge, we propose a new kernel-based attribution method. This approach effectively captures complex task interactions and does not rely on specific assumptions about the training process. This makes it broadly applicable to emerging paradigms, including in-context learning and reinforcement learning. Furthermore, we introduce an efficient gradient approximation technique that avoids expensive retraining, ensuring our method is scalable.
We validate our method's effectiveness with experiments across several distinct domains. 
The experimental results demonstrate that our approach provides more accurate attribution than existing methods across a diverse range of settings relevant to modern AI systems.

\section*{Acknowledgment}

Thanks to Chris R\'e for discussions about this work.
We appreciate the anonymous referees for their comments and suggestions.

This work is supported in part by NSF award IIS-2412008 and a startup grant from Northeastern University.

\bibliography{refs}

@inproceedings{koh2017understanding,
  title={Understanding black-box predictions via influence functions},
  author={Koh, Pang Wei and Liang, Percy},
  booktitle={International conference on machine learning},
  pages={1885--1894},
  year={2017},
  organization={PMLR}
}

@article{liidentification,
  title={Identification of Negative Transfers in Multitask Learning Using Surrogate Models},
  author={Li, Dongyue and Nguyen, Huy and Zhang, Hongyang R.},
  year={2023},
  journal={Transactions on Machine Learning Research}
}

@inproceedings{deng2024multi,
  title={Multi-group learning for hierarchical groups},
  author={Deng, Samuel and Hsu, Daniel},
  booktitle={Proceedings of the 41st International Conference on Machine Learning},
  pages={10440--10487},
  year={2024}
}

@article{rajeswaran2019meta,
  title={Meta-learning with implicit gradients},
  author={Rajeswaran, Aravind and Finn, Chelsea and Kakade, Sham M and Levine, Sergey},
  journal={Advances in neural information processing systems},
  volume={32},
  year={2019}
}

@inproceedings{ilyas2022datamodels,
  title={Datamodels: Predicting predictions from training data},
  author={Ilyas, Andrew and Park, Sung Min and Engstrom, Logan and Leclerc, Guillaume and Madry, Aleksander},
  booktitle={International conference on machine learning},
  year={2022}
}

@article{pruthi2020estimating,
  title={Estimating training data influence by tracing gradient descent},
  author={Pruthi, Garima and Liu, Frederick and Kale, Satyen and Sundararajan, Mukund},
  journal={Advances in Neural Information Processing Systems},
  volume={33},
  pages={19920--19930},
  year={2020}
}

@inproceedings{park2023trak,
  title={TRAK: Attributing Model Behavior at Scale},
  author={Park, Sung Min and Georgiev, Kristian and Ilyas, Andrew and Leclerc, Guillaume and Madry, Aleksander},
  booktitle={International Conference on Machine Learning},
  pages={27074--27113},
  year={2023},
  organization={PMLR}
}

@article{choe2024your,
  title={What is your data worth to gpt? llm-scale data valuation with influence functions},
  author={Choe, Sang Keun and Ahn, Hwijeen and Bae, Juhan and Zhao, Kewen and Kang, Minsoo and Chung, Youngseog and Pratapa, Adithya and Neiswanger, Willie and Strubell, Emma and Mitamura, Teruko and Schneider, Jeff and Hovy, Eduard and Grosse, Roger and Xing, Eric},
  journal={arXiv preprint arXiv:2405.13954},
  year={2024}
}

@inproceedings{li2024scalable,
  title={Scalable Fine-tuning from Multiple Data Sources: A First-Order Approximation Approach},
  author={Li, Dongyue and Zhang, Ziniu and Wang, Lu and Zhang, Hongyang R.},
  booktitle={Findings of the Association for Computational Linguistics: EMNLP 2024},
  pages={5608--5623},
  year={2024}
}

@inproceedings{bae2024training,
title={Training Data Attribution via Approximate Unrolling},
author={Juhan Bae and Wu Lin and Jonathan Lorraine and Roger Grosse},
booktitle={The Thirty-eighth Annual Conference on Neural Information Processing Systems},
year={2024},
}

@inproceedings{basu2021influence,
  title={Influence Functions in Deep Learning Are Fragile},
  author={Basu, Samyadeep and Pope, Phil and Feizi, Soheil},
  booktitle={International Conference on Learning Representations},
  year={2021}
}

@article{grosse2023studying,
  title={Studying large language model generalization with influence functions},
  author={Grosse, Roger and Bae, Juhan and Anil, Cem and Elhage, Nelson and Tamkin, Alex and Tajdini, Amirhossein and Steiner, Benoit and Li, Dustin and Durmus, Esin and Perez, Ethan and Hubinger, Evan and Lukošiūtė, Kamilė and Nguyen, Karina and Joseph, Nicholas and McCandlish, Sam and Kaplan, Jared and Bowman, Samuel R.},
  journal={arXiv preprint arXiv:2308.03296},
  year={2023}
}

@inproceedings{foret2021sharpnessaware,
    title={Sharpness-aware Minimization for Efficiently Improving Generalization},
    author={Pierre Foret and Ariel Kleiner and Hossein Mobahi and Behnam Neyshabur},
    booktitle={International Conference on Learning Representations},
    year={2021},
}

@article{
zhang2024noise,
title={Noise Stability Optimization for Finding Flat Minima: A Hessian-based Regularization Approach},
author={Hongyang R. Zhang and Dongyue Li and Haotian Ju},
journal={Transactions on Machine Learning Research},
year={2024},
}

@article{krizhevsky2009learning,
  title={Learning multiple layers of features from tiny images},
  author={Krizhevsky, Alex and Hinton, Geoffrey},
  year={2009},
  publisher={Toronto, ON, Canada}
}

@article{zhang2025linear,
  title={Linear-Time Demonstration Selection for In-Context Learning via Gradient Estimation},
  author={Zhang, Ziniu and Zhang, Zhenshuo and Li, Dongyue and Wang, Lu and Dy, Jennifer and Zhang, Hongyang R.},
  journal={Empirical Methods in Natural Language Processing},
  year={2025}
}

@article{garg2022can,
  title={What can transformers learn in-context? a case study of simple function classes},
  author={Garg, Shivam and Tsipras, Dimitris and Liang, Percy S and Valiant, Gregory},
  journal={Advances in neural information processing systems},
  volume={35},
  pages={30583--30598},
  year={2022}
}

@inproceedings{kwon2024datainf,
  title={DataInf: Efficiently Estimating Data Influence in LoRA-tuned LLMs and Diffusion Models},
  author={Kwon, Yongchan and Wu, Eric and Wu, Kevin and Zou, James},
  booktitle={The Twelfth International Conference on Learning Representations},
  year={2024},
}

@inproceedings{yu2020metaworld,
  title={Meta-world: A benchmark and evaluation for multi-task and meta reinforcement learning},
  author={Yu, Tianhe and Quillen, Deirdre and He, Zhanpeng and Julian, Ryan and Hausman, Karol and Finn, Chelsea and Levine, Sergey},
  booktitle={Conference on robot learning},
  pages={1094--1100},
  year={2020},
  organization={PMLR}
}

@article{engstrom2025optimizing,
  title={Optimizing ml training with metagradient descent},
  author={Engstrom, Logan and Ilyas, Andrew and Chen, Benjamin and Feldmann, Axel and Moses, William and Madry, Aleksander},
  journal={arXiv preprint arXiv:2503.13751},
  year={2025}
}

@article{ilyas2025magic,
  title={MAGIC: Near-Optimal Data Attribution for Deep Learning},
  author={Ilyas, Andrew and Engstrom, Logan},
  journal={arXiv preprint arXiv:2504.16430},
  year={2025}
}

@inproceedings{
deng2025a,
title={A Versatile Influence Function for Data Attribution with Non-Decomposable Loss},
author={Junwei Deng and Weijing Tang and Jiaqi W. Ma},
booktitle={Forty-second International Conference on Machine Learning},
year={2025}
}

@article{yang2025precise,
  title={Precise high-dimensional asymptotics for quantifying heterogeneous transfers},
  author={Yang, Fan and Zhang, Hongyang R. and Wu, Sen and R{\'e}, Christopher and Su, Weijie},
  journal={Journal of Machine Learning Research},
  volume={26},
  number={113},
  pages={1--88},
  year={2025}
}

@inproceedings{malladi2023kernel,
  title={A kernel-based view of language model fine-tuning},
  author={Malladi, Sadhika and Wettig, Alexander and Yu, Dingli and Chen, Danqi and Arora, Sanjeev},
  booktitle={International Conference on Machine Learning},
  pages={23610--23641},
  year={2023},
  organization={PMLR}
}

@inproceedings{li2024scalablemultitask,
  title={Scalable Multitask Learning Using Gradient-based Estimation of Task Affinity},
  author={Li, Dongyue and Sharma, Aneesh and Zhang, Hongyang R},
  booktitle={Proceedings of the 30th ACM SIGKDD Conference on Knowledge Discovery and Data Mining},
  pages={1542--1553},
  year={2024}
}

@incollection{micchelli1984interpolation,
  title={Interpolation of scattered data: distance matrices and conditionally positive definite functions},
  author={Micchelli, Charles A},
  booktitle={Approximation theory and spline functions},
  pages={143--145},
  year={1984},
  publisher={Springer}
}

@book{scholkopf2002learning,
  title={Learning with kernels: support vector machines, regularization, optimization, and beyond},
  author={Sch{\"o}lkopf, Bernhard and Smola, Alexander J},
  year={2002},
  publisher={MIT press}
}

@inproceedings{min2022rethinking,
  title={Rethinking the Role of Demonstrations: What Makes In-Context Learning Work?},
  author={Min, Sewon and Lyu, Xinxi and Holtzman, Ari and Artetxe, Mikel and Lewis, Mike and Hajishirzi, Hannaneh and Zettlemoyer, Luke},
  booktitle={Proceedings of the 2022 Conference on Empirical Methods in Natural Language Processing},
  pages={11048--11064},
  year={2022}
}

@inproceedings{kreer2025bayesian,
  title={Bayesian Influence Functions for Scalable Data Attribution},
  author={Kreer, Philipp Alexander and Wu, Wilson and Adam, Maxwell and Furman, Zach and Hoogland, Jesse},
  booktitle={High-dimensional Learning Dynamics},
  year={2025}
}

@inproceedings{ribeiro2016should,
  title={``Why Should I Trust You?'' Explaining the predictions of any classifier},
  author={Ribeiro, Marco Tulio and Singh, Sameer and Guestrin, Carlos},
  booktitle={Proceedings of the 22nd ACM SIGKDD international conference on knowledge discovery and data mining},
  pages={1135--1144},
  year={2016}
}

@inproceedings{johnson1984extensions,
  title={Extensions of Lipshitz mapping into Hilbert space},
  author={Johnson, William B},
  booktitle={Conference modern analysis and probability, 1984},
  pages={189--206},
  year={1984}
}

@inproceedings{wuunderstanding,
  title={Understanding and Improving Information Transfer in Multi-Task Learning},
  author={Wu, Sen and Zhang, Hongyang R. and R{\'e}, Christopher},
  booktitle={International Conference on Learning Representations},
  year={2020}
}

@article{hu2025grass,
  title={GraSS: Scalable Influence Function with Sparse Gradient Compression},
  author={Hu, Pingbang and Melkonian, Joseph and Tang, Weijing and Zhao, Han and Ma, Jiaqi W},
  journal={arXiv preprint arXiv:2505.18976},
  year={2025}
}

@article{radhakrishnan2024mechanism,
  title={Mechanism for feature learning in neural networks and backpropagation-free machine learning models},
  author={Radhakrishnan, Adityanarayanan and Beaglehole, Daniel and Pandit, Parthe and Belkin, Mikhail},
  journal={Science},
  volume={383},
  number={6690},
  pages={1461--1467},
  year={2024},
  publisher={American Association for the Advancement of Science}
}

@article{poggio2002networks,
  title={Networks for approximation and learning},
  author={Poggio, Tomaso and Girosi, Federico},
  journal={Proceedings of the IEEE},
  volume={78},
  number={9},
  pages={1481--1497},
  year={2002},
  publisher={IEEE}
}

@inproceedings{scholkopf2001generalized,
  title={A generalized representer theorem},
  author={Sch{\"o}lkopf, Bernhard and Herbrich, Ralf and Smola, Alex J},
  booktitle={International conference on computational learning theory},
  pages={416--426},
  year={2001},
  organization={Springer}
}

@inproceedings{zhang2025scalable,
  title={Scalable multi-objective and meta reinforcement learning via gradient estimation},
  author={Zhang, Zhenshuo and Duan, Minxuan and Ye, Youran and Zhang, Hongyang R.},
  booktitle={AAAI},
  year={2026}
}

@article{dicker2017kernel,
  title={Kernel ridge vs. principal component regression: Minimax bounds and the qualification of regularization operators},
  author={Dicker, Lee H and Foster, Dean P and Hsu, Daniel},
  journal={Electronic Journal of Statistics},
  year={2017}
}

@article{wang2024alpine,
  title={Alpine: Unveiling the planning capability of autoregressive learning in language models},
  author={Wang, Siwei and Shen, Yifei and Feng, Shi and Sun, Haoran and Teng, Shang-Hua and Chen, Wei},
  journal={Advances in neural information processing systems},
  volume={37},
  pages={119662--119688},
  year={2024}
}

@article{wang2025benefits,
  title={Benefits and pitfalls of reinforcement learning for language model planning: a theoretical perspective},
  author={Wang, Siwei and Shen, Yifei and Sun, Haoran and Feng, Shi and Teng, Shang-Hua and Dong, Li and Hao, Yaru and Chen, Wei},
  journal={arXiv preprint arXiv:2509.22613},
  year={2025}
}

@inproceedings{kloster2014heat,
  title={Heat kernel based community detection},
  author={Kloster, Kyle and Gleich, David F},
  booktitle={Proceedings of the 20th ACM SIGKDD international conference on Knowledge discovery and data mining},
  pages={1386--1395},
  year={2014}
}

@article{wu2023connecting,
  title={Connecting Pre-trained Language Model and Downstream Task via Properties of Representation},
  author={Wu, Chenwei and Lee, Holden and Ge, Rong},
  journal={Advances in Neural Information Processing Systems},
  volume={36},
  pages={47216--47238},
  year={2023}
}

@inproceedings{brown2021memorization,
  title={When is memorization of irrelevant training data necessary for high-accuracy learning?},
  author={Brown, Gavin and Bun, Mark and Feldman, Vitaly and Smith, Adam and Talwar, Kunal},
  booktitle={Proceedings of the 53rd annual ACM SIGACT symposium on theory of computing},
  pages={123--132},
  year={2021}
}
\bibliographystyle{iclr2026_conference}
\newpage
\appendix
\section{Complete Proofs}\label{app_proofs}

\paragraph{Derivation of influence functions.}
Let $S = \set{(x_1, y_1), (x_2, y_2), \dots, (x_n, y_n)}$ be a dataset including $n$ samples drawn independently from an unknown data distribution.
Let $f_W$ denote a model with parameters $W\in\real^d$.
Let $\hat L(f_W)$ denote the empirical loss of the model $f_W$ on $S$, averaged over the $n$ training data samples.
The influence function \citep{koh2017understanding}, which is defined by how much a trained model changes after adding or removing one sample $z = (x, y)$, is given by
\[ [\nabla^2 \hat L(f_W)]^{-1} \nabla \ell(f_W(x), y). \]

To derive this result, let the perturbed parameter after adding input $z$ with step size $\epsilon$ be given by:
\begin{align} 
    \widehat W_{\epsilon, z} = \arg \min_{W} \left( \hat{L}(f_W) + \epsilon \cdot \ell(f_W(x), y) \right).\label{eq_delta}
\end{align}
Define the parameter change $\Delta_\epsilon = \widehat W_{\epsilon, z} - \widehat W$ and note that, since $\widehat W$ does not depend on $\epsilon$, we can thus write:
\[ \frac{d \widehat W_{\epsilon, z}}{d \epsilon} = \frac{d \Delta_\epsilon} {d \epsilon}. \]
Based on first-order optimality conditions, from equation \eqref{eq_delta}, we get:
\[ \nabla \hat{L}(f_{\widehat W_{\epsilon, z}}) + \epsilon \cdot \nabla \ell(f_{\widehat W_{\epsilon, z}}(x), y) = 0. \]
Since $\widehat W_{\epsilon, z} \rightarrow \widehat W$ as $\epsilon \rightarrow 0$, we perform Taylor expansion centered at $\widehat W$ up to first-order as:
\begin{align*}
    \left( \nabla \hat{L}(f_{\widehat W}) + \epsilon \cdot \nabla \ell(f_{\widehat W_{\epsilon, z}}(x), y) \right) + \left( \nabla^2 \hat{L}(f_{\widehat W}) + \epsilon \cdot \nabla^2 \ell(f_{\widehat W}(x), y) \right) \Delta_\epsilon \approx 0,
\end{align*}
which is approximately zero after dropping the second-order term.

After setting $\epsilon$ to approaching zero, and solving for $\Delta_{\epsilon}$, we get:
\begin{align*}
    \Delta_{\epsilon} = - [ \nabla^2 \hat{L}(f_{\widehat W})]^{-1} \left( \nabla \hat{L}(f_{\widehat W}) + \epsilon \nabla \ell(f_{\widehat W}(x), y) \right)
    = - \epsilon [ \nabla^2 \hat{L}(\widehat W) ]^{-1} \nabla \ell(f_{\widehat W}(x), y),
\end{align*}
since $\nabla \hat{L}(f_{\widehat W}) = 0$.
Through the chain rule, the influence on any differentiable function $F(\mathbf s)$ becomes:
\begin{align*}
    \mathcal{I}_z(\mathbf s) = [\nabla_W F(\mathbf s)]^{\top} \big[\nabla_W^2 \hat L(f_W)\big]^{-1} \nabla \ell(f_W(x), y).
\end{align*}
By aggregating over $k = 1, 2, \dots, K$, we conclude the proof of equation \eqref{eq_if}.

\paragraph{Notations.} Given a square matrix $X \in \real^{d \times d}$, we use $\diag{X} \in \real^d$ to denote the diagonal entries of $X$ organized as a vector.
Let $\bignorms{X}$ denote the largest singular value of $X$.
Let $\bignorm{x}$ denote the Euclidean norm of a vector $x\in\real^d$.

We follow the convention of big-O notations. We use $f(n) = O(g(n))$ to indicate that there exists a constant $c > 0$ so that for large enough $n$, $f(n) \le c \cdot g(n)$.
We also write $f(n) \lesssim g(n)$ to indicate that $f(n) = O(g(n))$.

\subsection{Proof of Proposition~\ref{prop_second_order_datamodels}}\label{proof_second_order_datamodels}

\begin{proof}
We denote by $\mu=\mathbb{E}[\mathbf{s}] = [{p}]_K$.
First, we derive the population ordinary least squares (OLS) coefficients for the linear predictor
\[ g(\mathbf{s}) = \alpha + \sum_{j=1}^K \beta_j \mathbf{s}_j. \]
Let $\alpha^{\star}$ and $\beta^\star = (\beta^\star_1, \dots, \beta^\star_K)$ denote the minimizers of the population squared loss
\[ \alpha^\star, \beta^\star \leftarrow \arg\min_{\alpha, \beta} \exarg{\mathbf{s}\sim\mathcal{D}}
{\big(F(\mathbf{s}) - \alpha - \sum_{j=1}^K \beta_j \mathbf{s}_j \big)^2}. \]

Taking the derivative with respect to $\alpha$ and setting it to zero gives
\begin{align}\label{eq_alpha}
    \mathbb{E}\Big[F(\mathbf{s}) - \alpha^\star - \sum_{j=1}^K \beta^\star_j \mathbf{s}_j\Big] = 0,
\end{align}
We differentiate with respect to each $\beta^\star_i$ and set the derivative to zero:
\[ \mathbb{E}\Big[\mathbf{s}_i \big( F(\mathbf{s}) - \alpha^\star - \sum_{j=1}^K \beta^\star_j \mathbf{s}_j \big) \Big] = 0,
\qquad i = 1, \dots, K. \]

Expanding the expectation and using linearity yields
\begin{align} 
    \mathbb{E}[\mathbf{s}_i F(\mathbf{s})] - \alpha^\star \mathbb{E}[\mathbf{s}_i] - \sum_{j=1}^K \beta^\star_j\, \mathbb{E}[\mathbf{s}_i \mathbf{s}_j] = 0.  \label{eq_linear}
\end{align}
From equation \eqref{eq_alpha}, we have
\begin{align}
    \alpha^\star = \exarg{\mathbf{s}\sim\cD}{F(\mathbf{s}) } - \sum_{j=1}^K \beta^\star_j \exarg{\mathbf{s}\sim\cD}{\mathbf{s}_j}. \label{eq_alpha2}
\end{align}
By multiplying both sides of equation \eqref{eq_alpha2} with $\mathbf{s}_i$ and taking expectation on $\mathbf{s}$ again, we get
\begin{align}
    \ex{\mathbf{s}_i} \ex{F(\mathbf{s})} = \sum_{j=1}^K \beta_j \ex{\mathbf{s}_i} \ex{\mathbf{s}_j}. \label{eq_alpha3}
\end{align}
Subtracting and adding both sides from equation \eqref{eq_alpha3} %
into equation \eqref{eq_linear} leads to the following covariance form
\[ \mathrm{Cov}[\mathbf{s}_i, F(\mathbf{s})] = \sum_{j=1}^K \beta^\star_j\, \mathrm{Cov}[\mathbf{s}_i, \mathbf{s}_j], \qquad i=1,\dots,K, \]
which are the population normal equations. Since subsets are independent, we have
\[ \mathrm{Cov}[\mathbf{s}_i, \mathbf{s}_j] = 0 \quad \text{for } i\neq j, \]
so the covariance matrix of $\mathbf{s}$ is diagonal.  
Thus each normal equation decouples:
\[ \mathrm{Cov}[\mathbf{s}_k, F(\mathbf{s})] = \beta^\star_k\, \mathrm{Var}[\mathbf{s}_k]. \]

Solving for $\beta^\star_k$ gives the uni-variate OLS coefficient
\begin{align}
    \hat{\beta}^\star_k = \frac{\mathrm{Cov}[\mathbf{s}_k, F(\mathbf{s})]}{\mathrm{Var}[\mathbf{s}_k]}. \label{eq_beta_hat}
\end{align}

Next, we decompose the covariance using the second-order Taylor expansion as:
\begin{align*} 
    \mathrm{Cov}[\mathbf{s}_k,F(\mathbf{s})] 
    =& \underbrace{\mathrm{Cov}[\mathbf{s}_k, F(\mathbf{s}^\star)]}_{A_1}
    +\underbrace{\mathrm{Cov}\big[\mathbf{s}_k, \mathbf{G}_{\bs}^\top(\mathbf{s}-\mathbf{s}^\star)\big]}_{A_2}
    +\underbrace{\mathrm{Cov}\big[\mathbf{s}_k,\tfrac12(\mathbf{s}-\mathbf{s}^\star)^\top \mathbf{H}_{\mathbf{s}}(\mathbf{s}-\mathbf{s}^\star)\big]}_{A_3}  \\
    & + R_3(\bs_k), 
\end{align*}
where we denote the gradient and Hessian as $\mathbf{G}_{\bs}, \mathbf{H}_{\bs}$, respectively, and the expansion error term as $R_3(\bs_k)$.
Notice that $A_1=0$.
Thus, $A_2$ is the first-order term, and $A_3$ is the second-order term.

Next, by the independence between different coordinates of $\mathbf{s}$, we have that \[ A_2 = g_k \mathrm{Var}[\mathbf{s}_k], \]
where $g_k = [\nabla_{\bs} F(\mathbf{s})]_k$.

For $A_3$, write $\mathbf{s}-\mathbf{s}^\star=(\mathbf{s}-\mu)+(\mu-\mathbf{s}^\star)$ and expand the Hessian product as:
\[ Q=\tfrac12 \Big((\mathbf{s}-\mu)^\top \mathbf{H}_{\mathbf{s}}(\mathbf{s}-\mu)+2(\mu-\mathbf{s}^\star)^\top \mathbf{H}_{\mathbf{s}}(\mathbf{s}-\mu)+(\mu-\mathbf{s}^\star)^\top \mathbf{H}_{\mathbf{s}}(\mu-\mathbf{s}^\star)\Big). \]
For the second term above,
\[ \mathrm{Cov}\big[\mathbf{s}_k,(\mu-\mathbf{s}^\star)^\top \mathbf{H}_{\mathbf{s}}(\mathbf{s}-\mu)\big]
=\mathrm{Var}[\mathbf{s}_k]\sum_{j=1}^{K} \mathbf{H}_{\mathbf{s}}(k,j)(\mu_j-\mathbf{s}^\star_j), \]
using symmetry of $\mathbf{H}_{\mathbf{s}}$, where $\mathbf{H}_\mathbf{s}(k, j)$ denotes the $(k, j)$-th entry of $\mathbf{H}_{\mathbf{s}}$.
For the first term,
\[ (\mathbf{s}-\mu)^\top \mathbf{H}_{\mathbf{s}}(\mathbf{s}-\mu)
=\sum_{i=1}^{K} \mathbf{H}_{\mathbf{s}}(i, i)(\mathbf{s}_i-\mu_i)^2 + 2\sum_{1\le i<j\le K}\mathbf{H}_{\mathbf{s}}(i, j)(\mathbf{s}_i-\mu_i)(\mathbf{s}_j-\mu_j). \]
All cross terms with $i\neq j$ vanish by independence and zero mean of $\mathbf{s}_j-\mu_j$, leaving
\[ \tfrac12\mathrm{Cov}\big[\mathbf{s}_k,(\mathbf{s}-\mu)^\top \mathbf{H}_{\mathbf{s}}(\mathbf{s}-\mu)\big]
=\tfrac12 \mathbf{H}_{\mathbf{s}}(k, k) \mathrm{Cov} \big[\mathbf{s}_k,(\mathbf{s}_k-\mu_k)^2\big]. \]
With $X_k=\mathbf{s}_k-\mu_k$, $\mathrm{Cov}[\mathbf{s}_k,X_k^2]=\mathbb{E}[X_k^3]$, and for Bernoulli$(p)$,
\[ \mathbb{E}[X_k^3]=p(1-p)(1-2p). \]
Hence
\begin{align} 
\mathrm{Cov}[\mathbf{s}_k,Q]
=\mathrm{Var}[\mathbf{s}_k]\sum_{j=1}^K \mathbf{H}_{\mathbf{s}}(k, j)(\mu_j-\mathbf{s}^\star_j)+\tfrac12 \mathbf{H}_{\mathbf{s}}(k, k) p(1-p)(1-2p), \label{eq_cov-alphaQ-detailed}
\end{align}
which yields
\begin{align} 
    \bigabs{\beta^\star_k - g_k - \sum_{j=1}^K \mathbf{H}_{\mathbf{s}}(k, j)(\mu_j-\mathbf{s}^\star_j) - \tfrac12 \mathbf{H}_{\mathbf{s}}(k, k)(1 - 2p) } 
    \le R_3 (\bs_k) / \mathrm{Var}[\bs_k]. \label{eq_one_coor}%
\end{align}
By applying equation \eqref{eq_one_coor} over $k = 1, 2, \dots, K$, we have shown that
\begin{align}
    \bignorm{\beta^{\star} - \nabla_{\bs} F(\bs^{\star}) - (p - 1) \mathbf{H}_{\bs} [\mathbf{1}]_K - \frac{1 - 2p} 2 \diag{\mathbf{H}_{\bs}} } 
    \le \sum_{k=1}^K R_3(\bs_k) / (p(1 - p)), \label{eq_beta_star}
\end{align}
where $\mathrm{Var}[\bs_k] = p(1 - p)$ since $\bs_k$ is a Bernoulli random variable with probability $p$.
Assuming the third-order derivatives of $F(\bs)$ are bounded by some small enough $c_3$, then, we have that
\begin{align} \sum_{k=1}^K R_3(\bs_k) = O\left(c_3 \ex{\bignorm{s - s^{\star}}^3}\right) \le O\left( 8 c_3 K^{3/2} (1 - p) \right). \label{eq_error_sum}
\end{align}
By applying equation \eqref{eq_error_sum} to equation \eqref{eq_beta_star}, we obtain an error term as $O(c_3 K^{3/2} p^{-1})$.

The last step is to bound the minimizer in terms of the empirical loss $\hat R(\alpha, \beta)$ and the population loss $R(\alpha, \beta)$.
Recall that $\hat \beta$ is the empirical risk minimizer on $m$ samples.
Since we are working in a noiseless setting, we have that
\begin{align} \label{eq_concen} 
    \bignorm{\hat \beta - \beta^{\star}} \le C \bignorm{\beta^{\star}} \kappa(\Sigma) \sqrt{\frac {K + \log (\delta^{-1})} {m}},
\end{align}
for some fixed constant $C > 0$.
This can be shown using standard OLS analysis and concentration bounds on the sample covariance matrix.
In the case of Bernoulli sampling, the population covariance matrix is $p(1-p)\id$, so $\kappa(\Sigma) = 1$.
Thus, by combining equation \eqref{eq_concen} with the Taylor expansion errors above, %
we have concluded the proof of equation \eqref{eq_beta-solution-detailed}.
\end{proof}

\begin{corollary}
    In the setting of Proposition \ref{prop_second_order_datamodels}, the intercept and the residual error of linear surrogate models are given by
    \begin{align}
        &\hat\alpha = \mathbb{E}[F(\mathbf{s})] - \hat\beta^\top \mu, \label{eq_inter}\\
        &\min_{\alpha,\beta}\mathbb{E}\big[(F(\mathbf{s})-\alpha-\beta^\top\mathbf{s})^2\big]
        = \mathrm{Var}[Q] - \sum_{k=1}^K \frac{\mathrm{Cov}[\mathbf{s}_k,Q]^2}{\mathrm{Var}[\mathbf{s}_k]},\label{eq_res_err}
    \end{align}
    where $Q$ is defined in equation~\eqref{eq_cov-alphaQ-detailed}.
\end{corollary}

\begin{proof}
First, the intercept term \eqref{eq_inter} follows from equation \eqref{eq_alpha2}.
Next, consider the residual mean-squared error term \eqref{eq_res_err}.
Let $X=\mathbf{s}-\mu$ (zero mean, independent coordinates). The optimal linear predictor removes the constant and all components in the span of $\{X_k\}$. Therefore
\[ \mathcal{E}=\mathrm{Var} \Big[Q-\mathbb{E}[Q]-\sum_{k=1}^K b_k X_k\Big]
= \mathrm{Var}[Q]-\sum_{k=1}^K \frac{\mathrm{Cov}[X_k, Q]^2}{\mathrm{Var}[X_k]}. \]
Because $\mathrm{Cov}[X_k, Q] = \mathrm{Cov}[\mathbf{s}_k, Q]$ and $\mathrm{Var}[X_k] = \mathrm{Var}[\mathbf{s}_k] = p(1-p)$, this gives 
\begin{equation}
    \mathcal{E}\ \equiv\ \min_{\alpha,\beta}\mathbb{E} \left[(F(\mathbf{s}) - \alpha -\beta^\top\mathbf{s})^2\right]\ \approx\ \mathrm{Var}[Q]-\sum_{k=1}^K\frac{\mathrm{Cov}[\mathbf{s}_k,Q]^2}{\mathrm{Var}[\mathbf{s}_k]}. \label{eq_mse-master-detailed}
\end{equation}

It remains to compute $\mathrm{Var}[Q]$. Since $Q=\tfrac12 X^\top \mathbf{H}_{\mathbf{s}} X$,
\begin{align*}
    &X^\top \mathbf{H}_{\mathbf{s}} X = \sum_{i=1}^K \mathbf{H}_{\mathbf{s}}(i,i)X_i^2+2\sum_{1 \le i < j \le K}\mathbf{H}_{\mathbf{s}}(i,j)X_iX_j \\
    \quad\Rightarrow
    &\mathrm{Var}[X^\top \mathbf{H}_{\mathbf{s}} X]=\sum_i \mathbf{H}_{\mathbf{s}}(i,i)^2 \mathrm{Var}[X_i^2] + 4\sum_{1 \le i < j \le K}\mathbf{H}_{\mathbf{s}}(i,j)^2 \mathrm{Var}[X_i X_j],
\end{align*}
by independence and mean zero, eliminating mixed covariances. Thus
\[ \mathrm{Var}[Q]=\tfrac14\mathrm{Var}[X^\top \mathbf{H}_{\mathbf{s}} X]. \]
For Bernoulli$(p)$, $X_i\in\{1-p,-p\}$ with 
\[\mathbb{E}[X_i^2]=p(1-p),\] 
\[\mathbb{E}[X_i^4]=p(1-p)^4+(1-p)p^4. \]
Thus,
\begin{align*}
    \mathrm{Var}[X_i^2] &=\mathbb{E}[X_i^4]-\mathbb{E}[X_i^2]^2 =p(1-p)(1-2p)^2,\\
    \mathrm{Var}[X_iX_j]&=\mathbb{E}[X_i^2]\mathbb{E}[X_j^2]=(p(1-p))^2\ (i\neq j).
\end{align*}
Substituting back to $Q$ gives 
\begin{align}
    \mathrm{Var}[Q] &= \tfrac14 \left[\sum_{i=1}^K\mathbf{H}_{\mathbf{s}}(i,i)^2 p(1-p)(1-2p)^2+4\sum_{1 \le i < j \le K}\mathbf{H}_{\mathbf{s}}(i,j)^2 \big(p(1-p)\big)^2\right]. \label{eq_varQ-detailed}
\end{align}
Plugging equations \eqref{eq_cov-alphaQ-detailed} and \eqref{eq_varQ-detailed} into equation \eqref{eq_mse-master-detailed} gives the detailed residual mean-squared error.
\end{proof}

\subsection{Proof of Corollary~\ref{cor_if_approx}}\label{proof_cor_if_approx}

Next, we build on the above proof to derive Corollary \ref{cor_if_approx}.
We show that the influence function estimates the gradients of the performance function.
Our goal is to show that the vector of influence functions $\vec{ \mathcal{I}} = (\mathcal{I}_1,  \dots, \mathcal{I}_K)^\top$ is equal to the gradient vector $\nabla_\mathbf{s} F(\mathbf{s})$.

The optimal parameters $\hat{W}(\mathbf{s})=\argmin_W \hat{L}(f_{W},\mathbf{s})$.
The empirical risk minimizer corresponds to the uniform weight vector $\mathbf{s}^\star = [\mathbf{1}]_K$.
Let $F(\mathbf{s})$ be any differentiable performance metric that depends on the optimal parameters, e.g., the loss on a test sample $(x, y)$, $F(\mathbf{s}) = \ell(f_{\widehat{W}(\mathbf{s})}(x), y)$. 
The gradient of this function is given by $\nabla_\mathbf{s} F(\mathbf{s}^{\star})$.

\begin{proof}
Recall that the influence function uses an additive perturbation $\epsilon$ for a single task $i$. The perturbed objective is:
\begin{align*}
    \tilde L(f_{W}, \epsilon) = \left( \frac{1}{K} \sum_{j=1}^K \ell(f_W(x_j), y_j) \right) + \frac{\epsilon}{K} \ell(f_W(x_i), y_i).
\end{align*}
We can rewrite this expression by collecting terms for each sample:
\begin{align*}
    \tilde L(f_{W}, \epsilon) = \frac{1}{K} \left( (1+\epsilon) \ell(f_W(x_i), y_i) + \sum_{j \neq i} \ell(f_W(x_j), y_j) \right).
\end{align*}
Thus, the additive perturbation is a specific instance that corresponds to a path in the $K$-dimensional weight space $\mathbf{s}$:
\begin{align*}
    \mathbf{s}(\epsilon) = (1,1,\ldots, \underbrace{1+\epsilon}_{i\text{-th position}},\ldots,1)^{\top}
    = \mathbf{s}^{\star} + \epsilon \cdot e_{i}.
\end{align*}
where $e_i$ refers to the $i$-th standard basis vector.

The influence function for task $i$, $\mathcal{I}_i(\mathbf{s})$, is defined as the derivative of $F$ with respect to $\epsilon$, evaluated at $\epsilon=0$. We can compute this derivative using the multivariate chain rule on $F(\mathbf{s})$ as:
\begin{align*}
    \mathcal{I}_i(\mathbf{s}) = \left.\frac{d F(\mathbf{s}(\epsilon))}{d\epsilon}\right|_{\epsilon=0}.
\end{align*}
The chain rule states:
\begin{align}
    \frac{d F(\mathbf{s}(\epsilon))}{d\epsilon} = \big[\nabla_\mathbf{s} F(\mathbf{s}(\epsilon))\big]^\top \cdot \frac{d \mathbf{s}(\epsilon)}{d\epsilon}.  \label{eq_dF}
\end{align}
By definition, $\mathbf{s}(\epsilon) = \mathbf{s}^\star + \epsilon \cdot e_i$, thus
\[ \frac{d \mathbf{s}(\epsilon)}{d\epsilon} = e_i. \]
Substituting this back to equation \eqref{eq_dF}, we get:
\begin{align*}
    \frac{d F(\mathbf{s}(\epsilon))}{d\epsilon} = \big[\nabla_\mathbf{s} F(\mathbf{s}(\epsilon))\big]^\top e_i = \frac{\partial F(\mathbf{s}(\epsilon))}{\partial \mathbf{s}_i}.
\end{align*}
Thus, the derivative of $F(\mathbf{s}(\epsilon))$ with respect to $\epsilon$ is exactly the partial derivative with respect to the weight $\mathbf{s}_i$.
Finally, consider $\epsilon=0$. We have $\mathbf{s}(0)=\mathbf{s}^\star$. Therefore:
\begin{align*}
    \mathcal{I}_i(\mathbf{s}) 
    = \left.\frac{\partial F \left(\mathbf{s}(\epsilon)\right)}{\partial \epsilon}\right|_{\epsilon=0} 
    = \left.\frac{\partial F(\mathbf{s})}{\partial \mathbf{s}_i}\right|_{\mathbf{s} = \mathbf{s}^{\star}}.
\end{align*}
Since this holds for all $i=1, \dots, K$, the vector of influence functions is identical to the gradient vector of the performance function:
\begin{align*}
    \vec{\mathcal{I}} = [\mathcal{I}_1(\bs^{\star}), \dots, \mathcal{I}_K(\bs^{\star})]^\top = \nabla_\mathbf{s} F(\mathbf{s}^\star).
\end{align*}

Next, recall that $\bignorms{\mathbf{H}_{\mathbf{s}}} \le c_2$ by assumption.
As a result, the sampling bias and variance terms in equation \eqref{eq_beta-solution-detailed} are at most:
\begin{align*}
   c_2 \sqrt{K} + c_2 / 2.
\end{align*}

In summary, we have shown that $\bignorm{\beta^\star - \vec{\cI}} \le c_2 \sqrt{K} + c_2 / 2$.
Together with the concentration bound between $\hat\beta$ and $\beta^\star$ from equation \eqref{eq_concen}, we conclude that
\begin{align*}
    \bignorm{\hat \beta - \vec{\cI}} \le c_2 (\sqrt{K} + 1/2) + O\left(c_3 K^{3/2} p^{-1}\right) + O\left(\sqrt{\frac{K + \log(\delta^{-1})}{m}}\right),
\end{align*}
with probability at least $1 - \delta$ over the randomness of $m$ random subsets drawn from $\cD$,
which completes the proof of equation \eqref{eq_cor_if}.
\end{proof}

\subsection{Gradient estimation algorithms}\label{app_alg}

\textbf{Gaussian random projection.}
This method directly addresses the high dimensionality of the gradient by projecting it into a low-dimensional subspace. We employ a random matrix $P \in \mathbb{R}^{k \times d}$ (where $k \ll d$) to map the gradient $\nabla f_W(x) \in \mathbb{R}^{d}$ to a compressed representation \[ \nabla \tilde{f}_W(x) = P\nabla f_W(x) \in \mathbb{R}^{k}. \]
Motivated by the Johnson-Lindenstrauss Lemma, which guarantees that pairwise distances are approximately preserved, we then solve the perturbation objective using these low-dimensional features. The optimization is performed over a $k$-dimensional vector, making the problem tractable and independent of the original parameter count $d$.
\begin{algorithm}[t]
\caption{Gradient Estimation (\texttt{GradEx}) in the case of multi-class classification}\label{alg:compute_rp}
\textbf{Input:} Subset vector $\mathbf{s}^{(i)}$; (precomputed) logits $\{f_{{W}_0}\}$ and gradients $\{\nabla_{{W}} f_{W_0}\}$; test dataset $T_{\text{test}}$ and test loss $\ell$\\
\textbf{Requires:} $\ell_2$-regularization parameter $\lambda > 0$; random projection matrix $P \in \mathbb{R}^{k \times d}$\\
\textbf{Output:} Estimated performance on the test task %
\begin{algorithmic}[1]
\State Construct the sample subset $S_{\text{train}}^{(i)} = \{z_j = (x_j, y_j)\}_{j=1}^{|S|}$ from tasks selected in $\mathbf{s}^{(i)}$
\Statex /*\qquad\qquad\qquad\textit{Step 1: Project gradients and solve a low-dimensional regression problem}\hfill */
\State Compute projected gradients for all samples in the subset: $g_j = P{\nabla f_{W_0}(x_j)} $
\State Define the low-dimensional objective over $Z \in \mathbb{R}^k$:
    $$\tilde{L}(Z) = \sum_{j=1}^{|S|} \ell_{\text{CE}}\left( f_{{W}_0}(x_j) + \inner{g_j} {Z}, y_j \right) + \frac{\lambda}{2} \|Z\|_2^2$$
\State Compute the optimal low-dimensional regression by minimizing this convex objective:
    $$Z^\star = \argmin_{Z \in \mathbb{R}^k} \tilde{L}(Z)$$
\Statex /*\qquad\qquad\qquad\textit{Step 2: Estimate test performance} \hfill */
\State $y^{(i)} \leftarrow 0$
\For{each $z = (x, y) \in T_{\text{test}}$}
    \State Project the test gradient: $\tilde g = P\nabla f_{{W}_0}(x) $
    \State Approximate the test logits using the low-dimensional projections:
    $$\hat{f}(x) = f_{{W}_0}(x) + \inner{\tilde g} {Z^\star}$$
    \State Accumulate the loss: $y^{(i)} \leftarrow y^{(i)} + \ell(\hat{f}(x), y)$
\EndFor
\State $y^{(i)} \leftarrow \frac {y^{(i)} } {|T_{\text{test}}|}$ \Comment{Normalize the loss over the test dataset}
\State \textbf{return} $y^{(i)}$
\end{algorithmic}
\end{algorithm}

\textbf{Gradient estimation for in-context learning.}
For our experiments on the ICL task, we define the loss function and compute gradients in the embedding space. 
This approach enables us to avoid perturbation-based objective optimization methods, as we already know the difference in embeddings. 
Thus, we can directly perform a first-order estimation of the logit outputs for different prompts based on the gap between their embeddings.
Consider the model output of a prompt subset $S$ on an input $x$, denoted as $f_W(\phi(S, x))$.
Given an anchor prompt $S_0$, the first-order approximation of $f_W$ around the embedding vector $\phi(S_0, x)$ is given by:
\begin{align}
    f_W(\phi(S, x)) = f_W(\phi(S_0, x)) + \label{eq_fo}
    \inner{\nabla_{\phi}f_W(\phi(S_0, x))}{ \phi(S, x) - \phi(S_0, x)} + \epsilon_{S, x}.
\end{align}

\subsection{First-order approximation error}\label{app_foa}

First, kernel analyses of fine-tuning large transformers show that, in the kernel behavior regime, the model evolution is well-approximated by its first-order Taylor expansion around the pre-trained parameters $\widehat W$, such that
\[ f_W(x) = f_{W_0}(x) + \langle \nabla f_{W_0}(x), W - W_0 \rangle + \epsilon_{W}(x), \] with Taylor expansion residual error term $\epsilon_{W}(x)$ on $x$.
This justifies working in the linearized regime around $W_0$ even for highly nonlinear LLMs.

We show the approximation error on various models with parameters from 7B to 34B. We test on SST-2 and Coin flip datasets. Our findings in Table~\ref{tab_llm_approx_error} show that the errors remain negligible across different architectures and model sizes.

\begin{table}[t]
\centering
\caption{We report the approximation error on different sizes of models on SST-2 and Coin flip.}\label{tab_llm_approx_error}
\begin{tabular}{l|cc}
\toprule
              & SST-2                         & Coin flip                     \\ \hline
DeepSeek-LLM-7B & \multicolumn{1}{l}{$7.6_{\pm4.5} \times 10^{-3}$} & \multicolumn{1}{l}{$6.5_{\pm2.0} \times 10^{-3}$} \\
Llama-3.1-8B    & \multicolumn{1}{l}{$1.8_{\pm0.1} \times 10^{-2}$} & \multicolumn{1}{l}{$7.5_{\pm1.2} \times 10^{-3}$} \\
Llama-2-13B   & $7.7_{\pm1.2} \times 10^{-4}$ & $1.5_{\pm0.7} \times 10^{-3}$ \\
CodeLlama-34B & $7.4_{\pm2.7} \times 10^{-4}$ & $8.5_{\pm1.1} \times 10^{-4}$ \\
\bottomrule
\end{tabular}
\end{table}

Second, assuming the first-order approximation error is small, we show that the gradient estimation algorithm gives accurate estimates.
Let $\hat L(f_{W})$ be the empirical loss, assume that:
\begin{itemize}
    \item The average Taylor expansion residual is bounded by $\delta$:
    \[ \exarg{(x,y)}{ \big| f_W(x,y) - f_{W_0}(x,y) - [\nabla f_{W_0}(x,y)]^\top (W - W_0) \big| } \le \delta, \] 
    \item Gradients at $W_0$ are bounded by $G$, and
    \item The optimization is restricted to a domain of radius at most $D$, with a random projection that distorts inner products by at most $\epsilon$.
\end{itemize}

Denote by $\widehat W(S)$ the parameter produced by the gradient-based estimator (obtained from regression in the projected gradient feature space, as in our gradient estimation step).
One can show that
\[ \hat L(f_{\widehat W(S)}) \le \min_{W \in D} \hat L(f_{W}) + 2\delta + 4GD \epsilon. \]
Thus, provided that the linearization analysis ensures that the first-order approximation error $\epsilon_W(x)$ (and hence $\delta$) is small, the gradient-based estimator $\widehat W(S) = W_0 + Z_s^\star$ achieves training loss within $2\delta + 4GD\epsilon$ of the minimum loss in the search domain.
We summarize this discussion into the following proposition.

\begin{proposition}\label{prop_jl}
    Let $\cD \subseteq \real^d$ be a search space whose radius is at most $D$.
    Suppose the gradient of $f_{W_0}$ at the initialization $W_0$ in the training set is at most $G$ in Euclidean norm.
    For each task $i = 1, 2, \dots, n$, let $T_i$ denote the training data. Suppose that for every $i$,
    \begin{align*}
        {\frac 1 {\abs{T_i}} \sum_{(x, y) \in T_i} \abs{f_W(x, y) - f_{W_0}(x, y) - \nabla_W f_{W_0}(x, y)^{\top} (W - W_0) }} \le \delta.
    \end{align*}
    Provided that the random projection dimension $k$ satisfies $k = O\Big(\frac{\log N}{\epsilon^2}\Big)$, the training loss of ~$\widehat W(S)$ is bounded away from the minimum training loss for any $S \subseteq \set{1,2,\dots,n}$ as
    \begin{align}\label{eq_jl_gua}
        \hat L(f_{\widehat W(S)}) \le \min_{W\in\cD} \hat L(f_{W}) + 2\delta +  {4 G D } {\epsilon}.
    \end{align}
\end{proposition}

The proof is based on the Johnson-Lindenstrauss lemma \citep{johnson1984extensions}, which asserts that when $k = O\Big(\frac{\log N}{\epsilon^2}\Big)$, for any $g_i$ with $\bignorm{g_i} \le G$ and any $W, W_0$ in $\cD$, we have
\begin{align*}
    \abs{\inner{g_i}{W - W_0} - \inner{P g_i}{P (W  - W_0)}} 
    \le \epsilon \bigabs{\inner{g_i}{W - W_0}}
    \le {2 G D }{\epsilon}.
\end{align*}
The rest of the argument can be completed via simple inequalities.

\subsection{Expressivity of RBF kernels}\label{app_expressivity}

\begin{proposition}\label{prop_rbf}
Let $X = \{0,1\}^K$ and $k$ be the Gaussian RBF kernel. Then:
\begin{enumerate}
    \item[(i)] {Exact interpolation:} There exists a function $g \in \mathcal{H}_k$
    such that $g(\mathbf{s}) = F(\mathbf{s})$ for all $\mathbf{s} \in X$.
    \item[(ii)] {Universal approximation:} Consequently, for any $\epsilon > 0$,
    there exists $f \in \mathcal{H}_k$ such that
    \[ \max_{\mathbf{s} \in X} \bigl|F(\mathbf{s}) - f(\mathbf{s})\bigr| \le \epsilon. \]
\end{enumerate}
\end{proposition}

\begin{proof}
We first prove Part (i). The set $X$ consists of distinct points in $\mathbb{R}^K$.
It is a well-established fact that the Gaussian RBF kernel is \emph{strictly positive definite} on any finite set of distinct points~\citep{micchelli1984interpolation, scholkopf2002learning}.

Let $m = |X|$ be the number of tasks. We enumerate the elements of $X$ as $\mathbf{s}^{(1)}, \ldots, \mathbf{s}^{(m)}$ and let $y_j = F(\mathbf{s}^{(j)})$ be the target values. We define the Gram matrix $K_G \in \mathbb{R}^{m \times m}$ with entries \[ [K_{G}]_{i,j} = k(\mathbf{s}^{(i)}, \mathbf{s}^{(j)}). \]

Since the kernel is strictly positive definite, the matrix $K$ is strictly positive definite and therefore invertible (full rank). This implies that the linear system $K_G \theta = y$ has a unique solution $\theta \in \mathbb{R}^m$.
We construct the function
\[    g_\theta(\mathbf{s}) = \sum_{i=1}^m \theta_i k(\mathbf{\mathbf{s}}^{(i)}, \mathbf{\mathbf{s}}). \]
By construction, $g_\theta \in \mathcal{H}$ and satisfies $g_\theta(\mathbf{s}^{(j)}) = y_j = F(\mathbf{s}^{(j)})$ for all $j$. This proves Part~(i).

Part (ii) follows immediately from Part~(i). Since $g$ matches $F$ exactly on $X$, the approximation error is zero, which is always less than or equal to any $\epsilon > 0$.
\end{proof}

The intuition from the above result is that on $\{0,1\}^K$, two subsets that differ on only a few coordinates are close in Hamming distance, and intuitively their performances tend to be similar. The RBF kernel directly reflects this structure:
\[ k(\mathbf{s}, \mathbf{s}') = \exp\left(-\frac{1}{2\sigma^2} \|\mathbf{s} - \mathbf{s}'\|_2^2\right) \]
decays smoothly as the Hamming distance between $\mathbf{s}$ and $\mathbf{s}'$ increases. Thus, kernel regression assigns greater influence to nearby subsets and less to distant ones, capturing the local smoothness we observe in practice.

By contrast, polynomial kernels do not encode distance-based locality. A degree-$d$ polynomial kernel expands inputs into monomials (e.g., $s_i s_j$, $s_i s_j s_k, \ldots$). Thus, two subsets with a small Hamming distance are not necessarily mapped to nearby points in the feature space.

\section{Omitted Experiments}\label{app_exp}

We begin by describing details omitted from the experimental setup in Section \ref{sec_exp_setup}.

\textbf{Modular arithmetic tasks.}
We consider the following functions with $p = 97$: $a + b~(\mathrm{mod}\ p) = c$, and $a^2 + ab + b^2~(\mathrm{mod}\ p) = c$.
For each task, we generate a complete dataset by iterating through all possible pairs of $(a, b)$, where $a, b \in \{0, 1, \dots, p-1\}$. This results in a dataset of $97 \times 97 = 9,409$ unique equations for each function. 
Each equation is formatted as a sequence of tokens: \texttt{a}, \texttt{op}, \texttt{b}, \texttt{=}, \texttt{c}, where \texttt{op} is \texttt{+} for addition and a placeholder token for the quadratic function. 
We randomly split the dataset into a training set comprising $90\%$ of the data and a test set comprising the remaining $10\%$. 

We use a standard two-layer decoder-only transformer with embedding dimension $128$ as the classifier for all modular experiments. The model is trained to predict the correct token for the result \texttt{c}, given the preceding sequence (\texttt{a}, \texttt{op}, \texttt{b}, \texttt{=}). The cross-entropy loss and gradients are computed based on the model's output logits at the final token position.

To support our task attribution analysis, we use a grouped subset sampling strategy that partitions the training data based on the values of both operands, $a$ and $b$. We first partition the range of possible values for each operand, $\{0, 1, \dots, 96\}$, into 5 disjoint intervals. Specifically, a training example $(a, b, c)$ is assigned to a group $G_{i,j}$ where $i = \lfloor a / 20 \rfloor$ and $j = \lfloor b / 20 \rfloor$. This creates a $5 \times 5$ grid of 25 groups, each corresponding to a specific rectangular region in the input space.
Subsets of the training data are then constructed by selecting specific combinations of these groups. This method enables a fine-grained analysis of how the model learns across different regions of the input space, thereby allowing the systematic construction of task vectors. We sample $50$ subsets with a sampling ratio $0.9$, and use $40$ subsets as the surrogate training set.

\smallskip
\textbf{In-context learning.}
We use the \href{https://huggingface.co/Qwen/Qwen3-8B}{Qwen3-8B} as the base model and evaluate it on the sentiment classification dataset SST-2 and the reasoning task, the coin flip.
The \href{https://huggingface.co/datasets/nyu-mll/glue/viewer/sst2}{SST-2} dataset is a binary sentiment classification dataset comprising movie reviews labeled as either positive or negative, from the GLUE benchmark. The number of queries is $450$.
The \href{https://huggingface.co/datasets/skrishna/coin_flip}{Coin-Flip} dataset is an arithmetic reasoning task in which the model reads a natural-language description of a sequence of fair coin flips and must predict the final outcome (heads or tails). The number of queries is $869$.
For each dataset, we use the first $50$ candidate demonstrations as the data to be attributed. We sample $200$ subsets, each containing $4$ prompts, and use $80$ of these subsets as the surrogate training set.

\smallskip
\textbf{Multi-objective reinforcement learning.}
We use MT10 from the Meta-World benchmark~\citep{yu2020metaworld}, which consists of 10 diverse robotic manipulation tasks.
The agent's observation includes the environment state and a one-hot vector that specifies the current task. A sparse reward for moving the objective to its goal position.
The 10 tasks in MT10 are: {reach}, {push}, {pick-place}, {door-open}, {drawer-open}, {drawer-close}, {button-press-topdown}, {peg-insert-side}, {window-open}, and {box-open}.

For the task attribution evaluation, we use the Soft Actor-Critic (SAC) as the training algorithm and take the reward for each task as $F(\mathbf{s})$. We sample $50$ subsets, each containing $7$ tasks, and use $40$ subsets as the surrogate training set.

\smallskip
\textbf{Hessian-aware training regularizers.}
Sharpness-Aware Minimization (SAM) is an optimization method that aims to find {parameters lying in flat neighborhoods of the loss landscape}, rather than just minimizing the loss at a single point. Concretely, instead of minimizing the empirical loss $L_S(w)$, SAM minimizes the worst-case loss in an $\ell_p$-ball of radius $\rho$ around $w$: 
\[ \min_{w} L_S^{\mathrm{SAM}}(w) = \min_{w} \max_{\lVert \varepsilon \rVert_p \le \rho} L_S(w + \varepsilon). \]
Using a first-order Taylor approximation, the inner maximization has an approximate solution 
\[ \hat{\varepsilon}(w) = \rho   \frac{\nabla_w L_S(w)}{\lVert \nabla_w L_S(w) \rVert_2},\]
for the common case $p = 2$. Each SAM step then computes the gradient of the loss at the perturbed weights, $\nabla_w L_S(w + \hat{\varepsilon}(w))$, and updates $w$ in that direction, which biases training toward flatter minima that empirically generalize better than the sharp minima found by standard SGD.

NSO \citep{zhang2024noise} is an optimizer that explicitly encourages flat minima by minimizing a noise-perturbed loss and thereby regularizing the \emph{trace of the Hessian}. Given an empirical loss $f(W)$ and a zero-mean noise distribution $P$, NSO considers the smoothed objective \[ F(W) := \mathbb{E}_{U \sim P}[f(W + U)], \]
and uses a two-point noise injection scheme: at each step, it samples $U \sim P$ and averages gradients at symmetrically perturbed weights,
\[ G(W,U) = \frac 1 2 \left( g(W + U) + g(W - U) \right), \]
which cancels the first-order term in the Taylor expansion and keeps the second-order term \[ \frac{1}{2} U^\top \nabla^2 f(W) U. \]
For isotropic noise with covariance $\Sigma = \cN(0, \sigma^2 \id)$, this yields
\[ F(W) \approx f(W) + \frac{1}{2} \langle \Sigma, \nabla^2 f(W) \rangle. \]
Thus, running gradient descent on $F(\cdot)$ with Gaussian convolution is approximately the same as running gradient descent on $f(\cdot)$ plus a penalty term of $(\sigma^2 / 2)  \cdot \operatorname{Tr}[\nabla^2 f(W)]$.

\paragraph{Attribution baselines.}
Influence functions \citep{koh2017understanding} are proposed to measure the influence of individual samples and can also be adapted to estimate the influence of entire tasks by using a task-weighted loss function. The influence of infinitesimally up-weighting task $T_k$ on the test performance is given by:
\[ \mathcal{I}_k(\mathbf{s}) = [\nabla_W F(\mathbf{s})]^\top \big[\nabla^2_W L(f_{W}, \mathbf{s})\big]^{-1} \nabla_W \left( \frac{s_k}{\sum_{j=1}^K s_j}\ell_k(f_W, T_k) \right). \]
In our implementation, we approximate the inverse Hessian-vector product %
via conjugate gradient, which avoids direct matrix inversion.

TracIn \citep{pruthi2020estimating} traces the influence of training data through the optimization trajectory, avoiding Hessian computation. We adapt this method from the sample level to the task level by measuring the correlation between task gradients over the course of training. The influence of a training task $T_k$ on a test task $T_{\text{test}}$ is approximated by summing the dot products of their respective loss gradients across various training checkpoints:
\[ \text{TracIn}(T_k, T_{\text{test}}) = \sum_{t=1}^T \eta_t [\nabla f_{\text{test}}]^\top \nabla \ell_k(f_{W_t}, T_k), \]
where $W_t$ are the model parameters and $\eta_t$ is the learning rate at checkpoint $t$.

TRAK \citep{park2023trak} offers an efficient algorithm for data attribution by linearizing the model and using random projections. We adapt it to attribute influence at the task level. The core idea is to represent each task by an average of its constituent samples' projected gradients. Specifically, for each sample $z$ in a task, a feature vector is computed from the gradient of a model output function $f_W$. To obtain a feature vector for the entire task $T_k$, we average these features across all its samples. This task-level feature is then projected into a low-dimensional space using a random matrix $\mathbf{P}$.
The attribution score for the test task $T_{\text{test}}$ is then computed as:
$\phi(T_{\text{test}})^\top (\mathbf{\Phi}^\top \mathbf{\Phi})^{-1} \mathbf{\Phi}^\top \mathbf{Q}.$
Here, $\phi(T_{\text{test}}) \in \real^k$ is the projected feature vector for the test task, $\mathbf{\Phi}$ is the $K \times k$ matrix of stacked projected features for the $K$ training tasks, and $\mathbf{Q}$ is a $K \times K$ diagonal weighting matrix. The final scores are averaged over an ensemble of models to ensure robustness.

Linear surrogate models \citep{liidentification} learn a linear mapping from data weights to model predictions. Given a set of models trained on different data subsets, this approach fits a linear function $f: \{0,1\}^n \rightarrow \mathbb{R}$ that predicts test performance from subset indicators.
The coefficients of this linear model serve as attribution scores, capturing the marginal contribution of each training example across multiple training runs.

\begin{table}[t]
\small
\setlength{\tabcolsep}{6pt}
\renewcommand{\arraystretch}{1.15}
\caption{Top positive and negative influential prompt samples identified on the coin-flip task.}\label{tab_qualitative}
\begin{tabular}{p{0.28\textwidth} | p{0.30\textwidth} | p{0.30\textwidth}}
\toprule
\textbf{Query} &
\textbf{Top positive prompt samples} &
\textbf{Top negative prompt samples} \\
\midrule
Q: A coin is heads up. Kielmeyer \textbf{flips} the coin. Jevgenij \textbf{does not flip} the coin.  Is the coin still heads up? \textbf{(no)} &
Q: A coin is heads up. Erdener \textbf{flips} the coin. Ismari \textbf{does not flip} the coin.  Is the coin still heads up? \textbf{(no)} &
Q: A coin is heads up. Pachl \textit{flips} the coin. Lissett \textit{flips} the coin.  Is the coin still heads up? \textit{(yes)}\\ \hline
Q: A coin is heads up. Bonnitta \textbf{does not flip} the coin. Ellise \textbf{does not flip} the coin.  Is the coin still heads up? \textbf{(yes)}&
Q: A coin is heads up. Brittingham \textbf{does not flip} the coin. Alilet \textbf{does not flip} the coin.  Is the coin still heads up? \textbf{(yes)} &
Q: A coin is heads up. Kimyetta \textit{flips} the coin. Raynel \textbf{does not} flip the coin.  Is the coin still heads up? \textit{(no)} \\ \hline
Q: A coin is heads up. Roeland \textbf{does not flip} the coin. Joeliz \textbf{flips} the coin.  Is the coin still heads up? \textbf{(no)} &
Q: A coin is heads up. Kulju \textbf{does not flip} the coin. Afrodisio \textbf{flips} the coin.  Is the coin still heads up? \textbf{(no)} &
Q: A coin is heads up. Pachl \textit{flips} the coin. Lissett \textbf{flips} the coin.  Is the coin still heads up? \textit{(yes)} \\
\bottomrule
\end{tabular}
\end{table}

\subsection{Omitted experimental results}\label{app_correlation_datamodels_if}

\paragraph{Qualitative results.}
In Table~\ref{tab_qualitative}, we present the most positively and negatively prompt samples obtained by \algo{} on the Coin flip dataset.
We use bold to show information that matches the query, and italics for information that differs from the query.

\paragraph{Correlation between surrogate models and influence functions.}
We designed an experiment using an $\ell_2$-regularized logistic regression model (with an $\ell_2$ penalty of $10^{-2}$) on the Wisconsin Breast Cancer dataset. The data was standardized and split into a training set of 455 samples and a test set. Our analysis focused on quantifying the influence of each training point on the loss of a single, randomly selected test point.

We computed three distinct valuation scores for every training point: the ground-truth leave-one-out (LOO) score, the first-order influence function (IF) approximation, and the coefficients from a linear surrogate model. The LOO scores were obtained by exhaustively retraining the model from scratch (cold-start) after removing each training point individually and recording the change in test loss. The IF scores were calculated as a first-order approximation using the Hessian of the full training loss, stabilized with a $10^{-3}$ damping term. To build the surrogate model, we generated 1000 subsets, each containing 430 training points, and retrained a model on each subset to measure the resulting change in test loss. These loss changes were then used as targets in a final linear regression, whose coefficients, solved via ordinary least squares (OLS), provided the surrogate model scores.

\paragraph{Nonlinear numerical simulation.} %
We consider a binary classification task in a feature space $\mathcal{X} \in \mathbb{R}^2$ with labels $\mathcal{Y} \in \{0,1\}$.
The base learning algorithm, denoted by $\mathcal{A}$, is a two-layer MLP. The network architecture consists of an input layer, two hidden layers with 16 neurons each and ReLU activation functions, and a final output layer with a softmax activation.

To analyze the influence of different training samples, we construct a series of related but distinct training subsets. We first define a small, fixed set of $N$ anchor data points, $S_{\text{anchor}}=\{(x_i,y_i)_{i=1}^N\}$. Next, we define a candidate set of $K$ additional data points, $\mathcal{C} = \{ c_1, c_2, \dots,c_K \}$, which are sampled along the path that traverses a critical region near this boundary.
The experiment then focuses on a series of $K$ training subsets, where each subset $S_j$ is formed by combining the fixed anchor set with exactly one candidate point from $\mathcal{C}$:
$S_j=S_{\text{anchor}}\cup \{(c_j, y_c)\}, \quad \text{for } j = 1, \ldots, K$.
We select a fixed test point $x_{\text{test}}\in\mathcal{X}$ whose prediction is sensitive to the location of the decision boundary. The goal is to attribute the model's prediction on $x_{\text{test}}$ to the choice of the candidate point $c_j$.

\subsection{Ablation studies}\label{app_kernel_comparison}

Kernel methods often require greater sample complexity to converge to an optimal solution than linear models. We conduct an ablation study to investigate this trade-off.
We compared the performance of the kernel surrogate model with that of the linear surrogate model, varying the size of the data subset used for their construction, sampling from 20\% to 80\% of the total subsets.
The results in Figure~\ref{fig_ablation_sample_num} show that the kernel surrogate model consistently outperforms the linear surrogate model across all tested subset sizes. Notably, even when the number of samples was small (e.g., at the 20\% subset level), the kernel method still demonstrated a better performance.

\begin{figure}[t]
    \begin{subfigure}[b]{0.245\textwidth}
        \centering
        \includegraphics[width=0.995\textwidth]{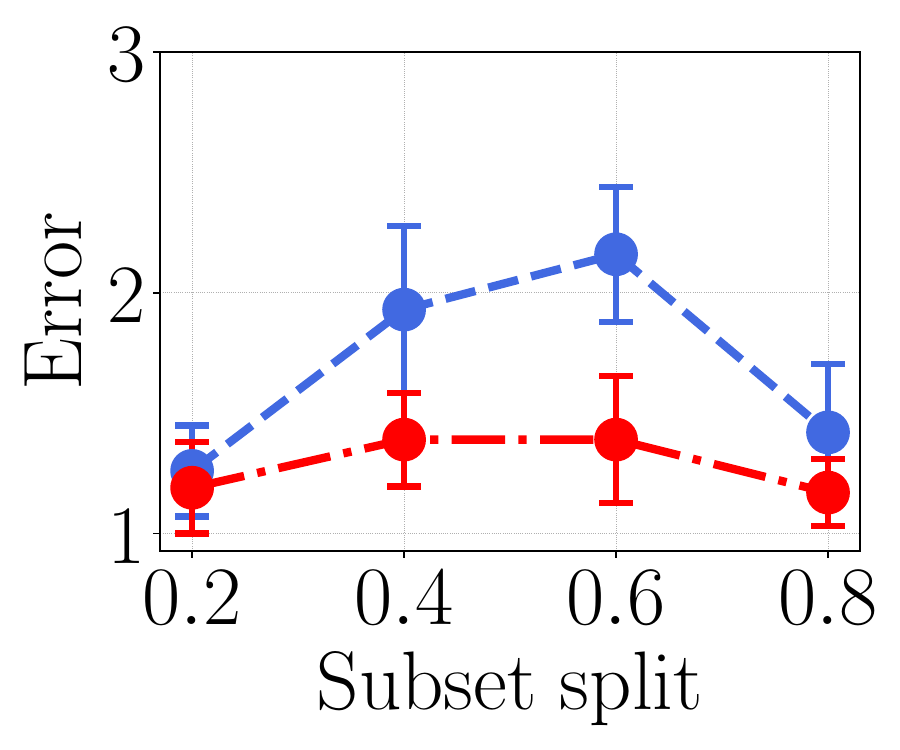}
        \caption{Modular addition}
    \end{subfigure}
    \begin{subfigure}[b]{0.245\textwidth}
        \centering
        \includegraphics[width=0.995\textwidth]{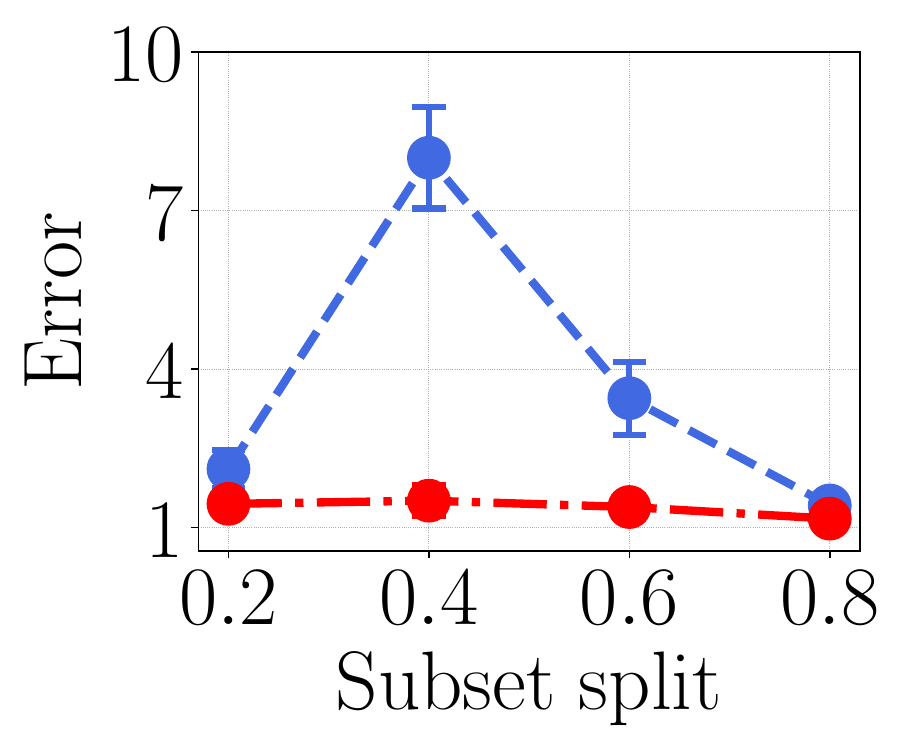}
        \caption{Modular quadratic}
    \end{subfigure}
    \begin{subfigure}[b]{0.245\textwidth}
        \centering
        \includegraphics[width=0.995\textwidth]{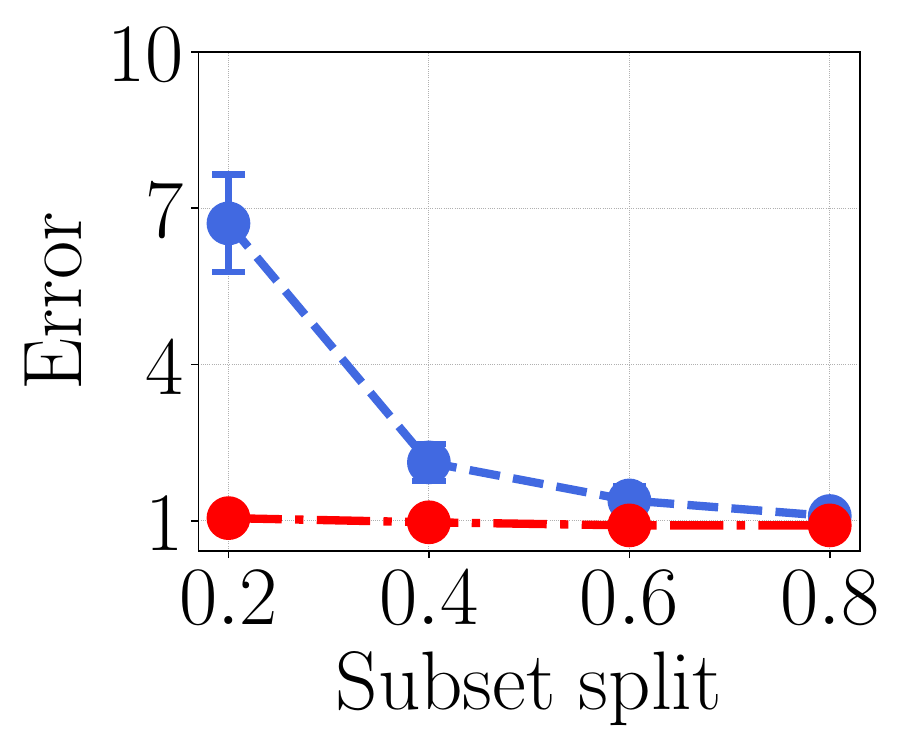}
        \caption{SST-2}
    \end{subfigure}
    \begin{subfigure}[b]{0.245\textwidth}
        \centering
        \includegraphics[width=0.995\textwidth]{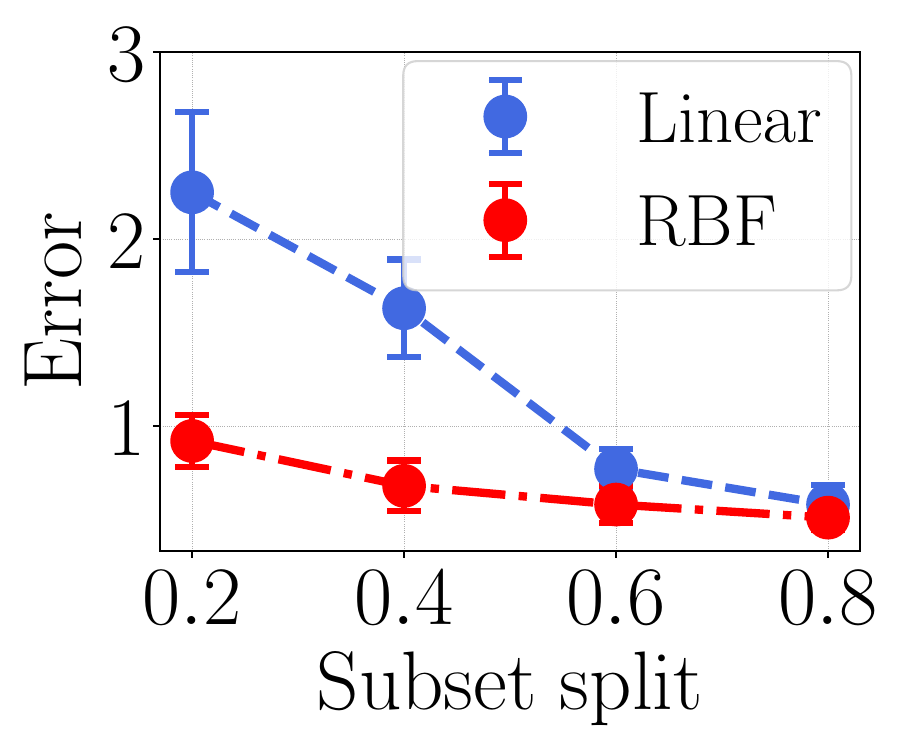}
        \caption{Coin flip}
    \end{subfigure}
    \caption{We investigated how the size of the surrogate training split affects the residual error of the linear and kernel models. We observed that across a range of training splits from $20\%$ to $80\%$, the kernel model consistently yields a lower residual error than the linear model.} \label{fig_ablation_sample_num}
\end{figure}

\paragraph{Hyperparameters analysis.}
We evaluate the sensitivity of our approach to different hyperparameter settings in the in-context learning task in Table~\ref{tab_ablation_lambda}.
Specifically, we vary $\lambda$ from $10^{-3}$ to $1$ and $\gamma$ from $10^{-5}$ to $10^{-1}$.
We find that our approach is relatively robust to changes in both $\lambda$ and $\gamma$, exhibiting stable performance across the entire range.

In our experiments, we set $\lambda=10^{-1}$ and $\gamma=1/n$ as the default configuration, where $n$ denotes the number of tasks.

\paragraph{Comparison of kernels.}
We use the CIFAR-10 dataset and the ResNet-9 network classifier for our experiments. The dataset consists of 60,000 $32\times32$ color images in ten classes, with 6,000 images per class. We utilize the standard training set of 50,000 images.
To align with the scope of our task-attribution methodology, we partition the 50,000 training samples into 50 disjoint tasks. Each task is constructed by randomly sampling 1,000 data points from the training set without replacement. This results in $50$ distinct tasks, which form the basis of our analysis. For the training of our surrogate model, we randomly select $30$ of these $50$ tasks.

We evaluate the performance of different kernels. Let $\mathbf{\mathbf{s}}^{(a)}$ and $\mathbf{\mathbf{s}}^{(b)}$ denote different subset indices. The kernels evaluated in our study are:
\begin{itemize}
    \item Polynomial kernel (Degree-$1$, $2$, $3$): The general form of the polynomial kernel is 
    \[ k(\mathbf{s}^{(a)},\mathbf{s}^{(b)}) = \big( (\mathbf{s}^{(a)})^\top\mathbf{s}^{(b)}+c\big)^d. \]
    In our experiments, we set the constant $c=0$ and test for degrees $d\in\{1,2,3\}$.
    
    \item Radial basis function (RBF) kernel: The RBF kernel is defined by the following equation:
    \[ k(\mathbf{s}^{(a)},\mathbf{s}^{(b)})= \exp\big(-\gamma \| \mathbf{s}^{(a)} - \mathbf{s}^{(b)} \|^{2}\big), \]
    where $\gamma$ is a parameter that controls the smoothness of the kernel function relative to the distance between $\mathbf{s}^{(a)}$ and $\mathbf{s}^{(b)}$.
\end{itemize}

In Table~\ref{tab_kernels_app}, we find that the linear surrogate model and degree-$1$ kernel have a large residual error. The degree-$2$ kernel achieves the lowest residual error among the three polynomial kernels. The RBF kernel achieves the lowest residual error among all the kernels.

\begin{table}[t]
\centering
\caption{We compare the performance of \algo{} under different $\lambda$ and $\gamma$.}\label{tab_ablation_lambda}
\resizebox{\textwidth}{!}{
\begin{tabular}{l|cccc | ccccc}
\toprule
      & $\lambda=1$ & $\lambda=10^{-1}$ & $\lambda=10^{-2}$ & $\lambda=10^{-3}$ & $\gamma=10^{-1}$ & $\gamma=10^{-2}$ & $\gamma=10^{-3}$ & $\gamma=10^{-4}$ & $\gamma=10^{-5}$  \\ \hline
Error & $0.81_{\pm0.01}$  & $0.65_{\pm0.01}$ & $0.73_{\pm0.04}$ & $0.86_{\pm0.10}$ & $0.65_{\pm0.01}$ & $0.66_{\pm0.01}$ & $0.87_{\pm0.01}$ & $1.01_{\pm0.03}$ & $1.02_{\pm0.03}$ \\
LDS   & $0.52_{\pm0.03}$ & $0.54_{\pm0.01}$ &  $0.51_{\pm0.02}$ &  $0.51_{\pm0.06}$ & $0.54_{\pm0.01}$ &  $0.53_{\pm0.02}$ &  $0.52_{\pm0.04}$ & $0.51_{\pm0.02}$ & $0.51_{\pm0.04}$ \\
\bottomrule
\end{tabular}}
\end{table}

\subsection{Extensions and discussions}

{Model-agnostic meta-learning} (MAML) seeks to learn a model initialization that is optimized for rapid adaptation to new tasks, typically with only a few gradient steps. Our work on task attribution, which measures a model's sensitivity to its training data, is connected to this goal. %
The error of first-order attribution methods is directly governed by the Hessian. For attribution, a large Hessian signifies a highly curved, non-linear performance landscape where linear approximations are unreliable, leading to significant attribution error. In contrast, MAML leverages this same curvature as a signal. The meta-gradient update in MAML involves differentiating through an inner-loop gradient step, a calculation that explicitly depends on the Hessian of the inner-loop loss with respect to the model parameters. %
There also exists a mathematical parallel between the derivation of influence functions and the formulation of implicit MAML (iMAML) \citep{rajeswaran2019meta}. Influence functions can be derived by applying the implicit function theorem to the first-order optimality condition of the perturbed loss, yielding an expression for the parameter change that involves the inverse Hessian. Analogously, iMAML uses the implicit function theorem to derive an analytical expression for the meta-gradient that depends only on the solution of the inner optimization, not the path taken to reach it.%

{Multi-group learning} aims to train a single predictor that performs robustly across a predefined set of subgroups, addressing the common failure mode where high average accuracy masks poor performance on critical sub-populations \citep{deng2024multi}. %
The influence of task $k$ on the performance of a model evaluated on task $j$ can be directly interpreted as the marginal contribution of group $k$'s training data to the model's performance on group $j$. A positive influence value for a loss-based metric provides a clear signal of negative transfer.
Standard influence functions, being first-order approximations, capture pairwise, additive effects and struggle to model more complex issues. For example, a model may perform poorly on an intersectional group (e.g., women of a specific race) due to higher-order interactions between the data subgroups that linear methods cannot detect. By using a non-linear RBF kernel, our model learns a global function that captures combinatorial effects between groups.
This can also inform the design of robust learning algorithms. 
For example, the MGL-Tree algorithm for hierarchical groups decides whether to use a general parent-level predictor or a specialized child-level predictor by comparing their empirical risks. Our method could provide a more principled, causal signal to guide this choice. %

\begin{table}[t!]
\centering
\caption{We investigate the surrogate model performance with different kernels. We run each experiment with five random seeds to report the standard deviations.
}\label{tab_kernels_app}
\begin{tabular}{l|ccccc}
\toprule
Residual error               & Linear model   & Degree-1        & Degree-2       & Degree-3 & RBF \\ \hline
CIFAR-10 & $4.4_{\pm0.9}$  & $3.8_{\pm0.8}$  & $1.6_{\pm0.1}$  & $2.7_{\pm0.2}$ & $1.0_{\pm0.0}$    \\
Modular arithmetic & $4.6_{\pm1.3}$  & $2.2_{\pm0.5}$  & $1.7_{\pm0.4}$  & $3.3_{\pm0.7}$ & $1.5_{\pm0.4}$   \\
In-context learning & $0.8_{\pm0.2}$  & $0.6_{\pm0.1}$  & $0.5_{\pm0.1}$  & $0.5_{\pm0.1}$ & $0.4_{\pm0.1}$\\
Multi-objective RL & $0.2_{\pm0.1}$  & $0.2_{\pm0.1}$  & $0.1_{\pm0.1}$  & $0.1_{\pm0.1}$ & $0.1_{\pm0.1}$\\
\bottomrule
\end{tabular}
\end{table}

\begin{resection}
\paragraph{Discussions.}
We now discuss this work with several related literature. 
There is an earlier line of work called LIME \citep{ribeiro2016should}, which addresses the question of explaining a single prediction at the feature level by fitting a surrogate that maps feature perturbations of a specific input to changes in the model output.

Our work instead focuses on training-data attribution by modeling how changing the weights of training examples affects the model's behavior. Although both approaches rely on local surrogate modeling, they operate on different objects. As a result, our work is complementary to this important line of work on LIME.

In terms of what’s new in this paper relative to this earlier line of work, our first observation is to connect linear surrogate models to influence functions; this is shown in Section 3.1, where we find the Ordinary Least Squares (OLS) regressor is approximately equal to influence functions~\citep{koh2017understanding}. This observation partially explains why linear surrogate models can perform well empirically~\citep{ilyas2022datamodels}.

At the same time, we identify a limitation of linear surrogate models in the context of task/data attribution, which motivates our kernel-based surrogate modeling. Kernels offer a universal function approximation framework, enabling the surrogate to capture nonlinear dependencies between training data and model behavior. To the best of our knowledge, no prior work in data attribution uses a kernel-based surrogate.

In the context of task attribution, kernel surrogate models are challenging to estimate due to the need for repeated training. Our second contribution is an efficient gradient-based estimator for fitting the kernel surrogate model. The computational cost consists of three parts: pre-computing $f_{\hat{W}}$, gradient estimation, and surrogate fitting. Except for the pre-computation step, which only involves computing gradients once on the training data, all the remaining computations can be executed quickly on CPUs.

As another remark, our in-context learning focuses on text classification or linear regression cases where the ordering of the examples does not affect the outcome.
This setting builds on the prior work of \cite{min2022rethinking}, who found that for many ICL classification tasks, model outputs are often not sensitive to permutations of the context examples. We verify this in our own experiments by evaluating the model under multiple random permutations of the same demonstration set: the standard deviation across different orders is only about $11\%$ of the mean loss, indicating that order effects are relatively minor in our setting. This justifies our choice to model only which demonstrations are included, rather than their specific ordering.

Beyond this setup, we agree that there are important scenarios, such as math reasoning with chain-of-thought style prompts, where the order of demonstrations plays a significantly larger role. Our framework can be naturally extended to handle such cases. Concretely, we can treat different orderings of the same demonstration set as different inputs to the surrogate model. For each permutation $\pi$ of a fixed set of examples, we would evaluate the corresponding performance $F(\mathbf{s},\pi)$ and fit the kernel surrogate on $(\mathbf{s},\pi)$, thereby assigning different scores to different orderings. In this way, our method can explicitly model and attribute order-sensitive effects when they are significant. This would be an interesting question for future work.
\end{resection}

\end{document}